\documentclass[11pt]{article}

\usepackage[UKenglish]{babel}
\usepackage{latexsym}
\usepackage{amssymb}
\usepackage{natbib}
\usepackage{amsmath}
\usepackage{color}
\usepackage[all]{xy}
\usepackage{url}
\usepackage{charter}
\usepackage{a4wide}
\usepackage[top=10mm, bottom=10mm, left=20mm, right=20mm]{geometry}
\usepackage{stmaryrd}

\newcommand{\bis}{\mathrel{\mathchoice%
{\raisebox{.3ex}{$\,
  \underline{\makebox[.7em]{$\leftrightarrow$}}\,$}}%
{\raisebox{.3ex}{$\,
  \underline{\makebox[.7em]{$\leftrightarrow$}}\,$}}%
{\raisebox{.2ex}{$\,
  \underline{\makebox[.5em]{\scriptsize$\leftrightarrow$}}\,$}}%
{\raisebox{.2ex}{$\,
  \underline{\makebox[.5em]{\scriptsize$\leftrightarrow$}}\,$}}}}

\newcommand{\grammaris}{\hspace{1mm}::=\hspace{1mm}}
\newcommand{\grammarsep}{\hspace{1mm}\mid\hspace{1mm}}

\newcommand{\SELA}{\ensuremath{\textbf{SIG}}}
\newcommand{\DEL}{\texttt{DEL}}

\newcommand{\ETL}{\texttt{ETL}}

\newcommand{\M}{{\mathcal{M}}}
\newcommand{\K}{{\mathcal{K}}}

\newcommand{\TURN}{{\mathtt{TURN}}}
\newcommand{\DET}{{\texttt{DET}}}
\newcommand{\KE}{{\texttt{KE}}}
\newcommand{\EXTURN}{{\texttt{EXTURN}}}
\newcommand{\INFO}{{\texttt{INFO}}}

\newcommand{\N}{{\mathcal{N}}}

\newcommand{\Act}{{\textbf{A}}}
\newcommand{\B}{{\textbf{B}}}
\newcommand{\BP}{{\textbf{P}}}
\newcommand{\lr}[1]{\langle #1\rangle}
\newcommand{\rel}[1]{\stackrel{#1}{\rightarrow}}
\newcommand{\tr}[1]{\text{#1}}
\newcommand{\lra}{\leftrightarrow}

\newcommand{\DISTK}{\ensuremath{\mathtt{DISTK}}}
\newcommand{\TAUT}{\ensuremath{\mathtt{TAUT}}}
\newcommand{\GEN}[1]{\ensuremath{\mathtt{NEC}[#1]}}
\newcommand{\GENK}{\ensuremath{\mathtt{NECK}}}

\newcommand{\MP}{\ensuremath{\mathtt{MP}}}

\newcommand{\AxTr}{\ensuremath{\mathtt{T}}}
\newcommand{\AxTrans}{\ensuremath{\mathtt{4}}}
\newcommand{\AxEuc}{\ensuremath{\mathtt{5}}}

\newcommand{\G}{\ensuremath{\mathcal{G}}}
\newcommand{\Ag}{\ensuremath{\mathbf{I}}}

\newcommand{\tot}{\ensuremath{\leftrightsquigarrow}}
\newcommand{\DIST}{\ensuremath{\mathtt{DIST}[a]}}

\newcommand{\U}{\mathcal{U}}

\newcommand{\MM}{\mathbb{M}}

\newcommand{\cic}{\circledast}
\newcommand{\Ue}{a}
\newcommand{\Uf}{b}
\newcommand{\LDEL}{\ensuremath{\textbf{LDEL}}}
\newcommand{\LEL}{\ensuremath{\textbf{LEL}}}

\newcommand{\NM}{\ensuremath{\mathtt{NM}}}
\newcommand{\PR}{\ensuremath{\mathtt{PR}}}

\newcommand{\PRp}{\ensuremath{\mathbf{Pr}}}
\newcommand{\NMp}{\ensuremath{\mathbf{Nm}}}
\newcommand{\F}{\mathcal{F}}

\renewcommand{\O}{O}

\newtheorem{theorem}{Theorem}
\newtheorem{definition}[theorem]{Definition}
\newtheorem{proposition}[theorem]{Proposition}

\newtheorem{example}{Example}
\newtheorem{remark}{Remark}
\newtheorem{lemma}[theorem]{Lemma}
\newenvironment{defn}[1] {\begin{definition}[#1]} {\hfill$\lhd$ \end{definition}}

\newenvironment{proof} {\textsc{Proof}\quad} {\hfill $\Box$\\ \medskip}

\title{From rules to runs: \\A dynamic epistemic take on imperfect information games} 

\author{Kai Li and Yanjing Wang\thanks{Corresponding author}
\thanks{The authors thank Aviad Heifetz, Tai-Wei Hu, Mamoru Kaneko, Yoram Moses, R. Ramanujam, the audience at SAET 2014, and the anonymous reviewer of this journal for their useful comments on the earlier versions of this work. Yanjing Wang acknowledges the support from ROCS of SRF by Education Ministry of China and NSSF key projects 15AZX020 and 12\&ZD119.}
\\Department of Philosophy, Peking University}

\date{}

\begin{document}
\maketitle
\begin{abstract}
In the literature of game theory, the information sets of extensive form games have different interpretations, which may lead to confusions and paradoxical cases. We argue that the problem lies in the mix-up of two interpretations of the extensive form game structures: \textit{game rules} or \textit{game runs} which do not always coincide. In this paper, we try to separate and connect these two views by proposing a dynamic epistemic framework in which we can compute the runs step by step from the game rules plus the given assumptions of the players. We propose a modal logic to describe players' knowledge and its change during the plays, and provide a complete axiomatization. We also show that, under certain conditions, the mix-up of the rules and the runs is not harmful due to the structural similarity of the two. 
\end{abstract}

\section{Introduction}
\subsection{Interpretations of information sets}

Game theory studies strategic interactions between rational players in the setting of games (cf. e.g., \cite{GTbook} for the basics). When the sequential structure of players' choices is important, games are usually represented in extensive forms, e.g., the tree-like structure of Game 1 below, which denotes a game where player 1 starts by choosing action $a$ or action $b$ followed by player 2's turn to choose between $c$ and $d$. Depending on their choices, the game may end up with different outcomes in favor of player 1 or 2. 

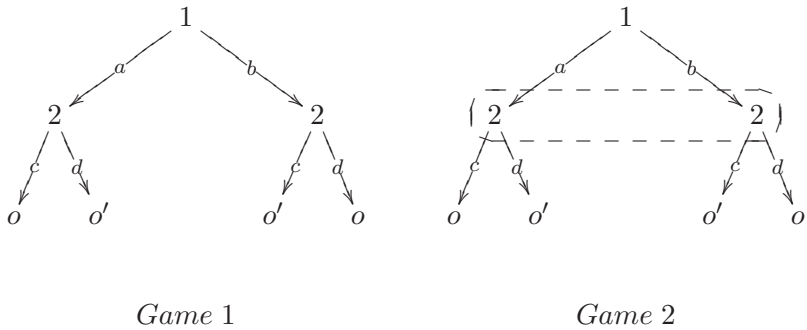
\begin{center}
$
\xymatrix@C-20pt{
&&&1\ar[dll]|a\ar[drr]|b\\
&2\ar[dl]|c\ar[dr]|d&&&&2\ar[dl]|c\ar[dr]|d\\
o& &o'&&o'&&o\\
&&& Game\  1 &&
}
$
\qquad 
$
\xymatrix@C-20pt{
&&&1\ar[dll]|a\ar[drr]|b\\
&2\ar[dl]|c\ar[dr]|d&&&&2\ar[dl]|c\ar[dr]|d\\
o& &o'&&o'&&o\\
\save "2,2"."2,6"!C="g1"*+[F--:<+20pt>]\frm{}
\restore
&&& Game\ 2&&
}
$
\end{center}

There can be two views regarding such a tree-like structure: 
\begin{itemize}
\item Game rules: an exact and exhaustive description of the physical rules of a game, e.g., if player 1 does $a$, then player 2 is allowed to choose either $c$ or $d$ which lead to outcomes $o$ or $o'$ respectively.   
\item Game runs: a temporal structure of the potential plays of the game, e.g., there are two possible plays $ac$ and $ad$ that start with action $a$ by player 1, and these two plays end up with $o$ and $o'$ respectively depending on the action by player 2. 
\end{itemize}
Essentially, game rules tell us what moves one player can make at which stage of the game, and game runs denote a global picture of the potential plays according to the rules. Most game theorists would agree on the rule-view 
although in many expositions of extensive form games the states in the game structures are essentially histories of actions, which fits the run-view better (cf. e.g., \cite{GTbook}). The run-view is also more intuitive to logicians who are familiar with temporal logic (cf. e.g., \cite{Halpern03}). In fact, in the settings like Game 1 where players always know `where they are', these two views on game structures do not make a significant difference.  

\medskip

In contrast with games like Game 1, there is an important class of extensive form games which are called games with  \textit{imperfect information}, represented usually by the tree-like structures with `bubbles' of states belonging to the same player, such as Game 2 above. The bubbles are called `information sets' which are supposed to capture the imperfect information that the players have when they need to make a decision: the players may not know the actual state of the game, e.g., in most card games you do not know others' hands. This lack of information may affect the ability of the players in obtaining the favorable outcomes of the game. For example, if player 2 prefers outcome $o$ over $o'$,  then in Game 1 she can always make sure $o$ by choosing $c$ or $d$ wisely according to player 1's earlier choice, but she cannot do so in Game 2, since she cannot distinguish the two decision points which require different actions to guarantee $o$. 

In different contexts, different interpretations are given to information sets. For example, in Game 2 above, the information set may mean: 
\begin{enumerate}
\item `Player 2 is not \textit{informed} where she is'
\item `Player 2 cannot \textit{distinguish} two histories $a$ and $b$' 
\item `Player 2 does not \textit{observe} the choice by player 1'
\end{enumerate}
The choice of the interpretation seems innocent at the first glance, and most game theorists view information sets as part of the game rules (cf. e.g, \cite{GTbook}): 
\begin{itemize}
\item Rule-view: information sets are \textit{primitive} physical rules which describe information given to the players by the game. 
\end{itemize}
However, if we look closer, only interpretation 1 above barely fits the idea of rules: player 2 may be given the same information at two different states of the game (by an  external game master). Interpretations 2 and 3 are more about the abilities of the players rather than physical rules: whether players can remember, infer or observe during the plays of the game. Moreover, there are two crucial concepts in such games which do not fit the rule view very well: \textit{uniform strategy} and \textit{perfect recall}. Uniform strategy requires a  player to choose uniformly among all the states in her  information set. Perfect recall describes a property of the players who, roughly speaking, do not forget what she has already distinguished in the past: if two states are not in the same information set of some player then their descendants in the game are also not in the same information set for the same player. However, although the player may be given the same information according to the rules, they may learn more by reasoning cleverly  during a play. In such a case, why should they stick to the uniform strategies? On the other hand, the information given to the players according to the rules has nothing to do with the ability to recall or to forget. A moment of reflection would take you closer to the following run-view where the above concepts make more sense:
\begin{itemize}
\item Run-view: information sets describe the actual uncertainty that the players have during the plays. 
\end{itemize}
However, is this rule-vs.-run distinction just a negligible conceptual difference? Can we keep overloading the interpretations of information sets? In fact, in the literature of game theory there have been indeed doubts about the interpretation of the information sets. Most notably, \cite{PR97b,PR97a} challenged the game theorists to explain some tricky examples of imperfect information games without perfect recall. Consider the following game with a player and nature (N) who makes a random move at the beginning to either $x_1$ or $x_2$ with 0.5 probability each. The player has to choose what to do next at $x_1, x_2$ and the information set $\{x_3, x_4\}$ (it has to be a uniform strategy): 
$$
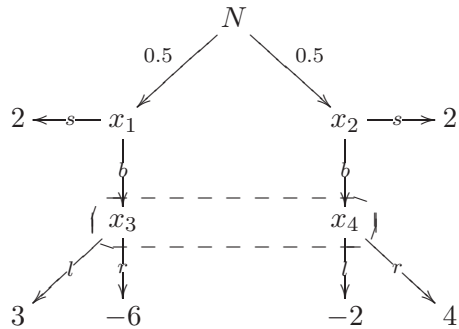
\xymatrix{
&& N\ar[dl]_{0.5}\ar[dr]^{0.5}\\
2&x_1\ar[l]|s\ar[d]|b& &x_2\ar[r]|s\ar[d]|b&2\\
&x_3\ar[dl]|l\ar[d]|r & & x_4\ar[d]|l\ar[dr]|r\\
3&-6&&-2&4
\save "3,2"."3,4"!C="g1"*+[F--:<+20pt>]\frm{}
\restore
}$$
It is not hard to calculate that to maximize the expected payoff (which is 3), the player should choose $s$ at $x_1$ and $b$ at $x_2$ and $r$ at the information set $\{x_3, x_4\}$. However, when playing the game, it seems that the player  may benefit by changing the strategy conditionally: at $x_1$, she may just change to do $b$ at $x_1$ and do $l$ at the information set $\{x_3, x_4\}$; if the player is at $x_2$ she can just stick to her original strategy. In the end this conditional strategy will give the player an expected payoff of $3.5$: something goes wrong!
As \cite{Halpern97Ab} correctly pointed out, by memorizing the strategy and its changes, the player is essentially mimicking the behavior of a perfect recall player. Essentially, the player may choose differently on the information set of $\{x_3, x_4\}$: they may be able to distinguish $x_3$ and $x_4$! The information set itself does not really capture the actual indistinguishability of the player. The problem lies in the fact that what is really known and not known during the play are not fully specified by the information sets. 

Actually, the problem is not always about imperfect recall, it may go wrong with perfect recall players. For example, in the following Game 3, although player 2 has the information set, if she knows that player 1 is rational and prefers strictly $o_1$ and $o_2$ to $o_3$ or $o_4$, then she is sure that she will end up with the left state in the information set. In this case she can tell where she is after player 1's move. As another example, Game 4 illustrates a game which may go on forever.\footnote{Although in classic game theory the game structures are tree-like structures based on partial histories, computer scientists are particular interested in the infinite games which can be represented finitely, which have to involve cycles.} It is reasonable to assume that player 2 cannot distinguish the two states for the first time (according to the rules on information), but will she have the same uncertainty after reaching the states in the information sets again by doing something herself? It depends on our assumption about the player but not the game rules. 
$$
\xymatrix@C-20pt{
&&&1\ar[dll]|a\ar[drr]|b\\
&2\ar[dl]|c\ar[dr]|d&&&&2\ar[dl]|c\ar[dr]|d\\
o_1& &o_2&&o_3&&o_4\\
\save "2,2"."2,6"!C="g1"*+[F--:<+20pt>]\frm{}
\restore
&&& Game\ 3&&
}
\qquad 
\xymatrix@C-20pt{
&&&1\ar[dll]|a\ar[drr]|b\\
&2\ar@/^/@(lu,ld)|c\ar[dr]|d&&&&2\ar[dl]|c\ar@/_/@(ru,rd)|d\\
& &o_1&&o_2&&\\
\save "2,2"."2,6"!C="g1"*+[F--:<+20pt>]\frm{}
\restore
&&& Game\ 4&&
}
$$

By now we have argued that although game theorists usually take the rule-view of the imperfect information games, the concepts related to information sets are best explained in terms of runs, and the two do not always coincide. 


\subsection{The basic ideas behind our proposal}

We propose to \textit{separate} rules and runs explicitly, and study \textit{how} to compute from rules to runs with assumptions of players' information and ability. As we will see, occasionally rules and runs do coincide structurally, and this may explain why many people implicitly embrace the mix-up interpretation. In this work, we focus on the perfect recall players but the framework can be generalized to imperfect recall players, to which we will come back at the end of the paper. 

\medskip


To start with, we separate game rules and the player assumptions, and further break down each part into intuitive explicit components: 
\begin{itemize}
\item Game rules: 
\begin{itemize}
\item Game arena: at which state what actions can be done by whom with what outcomes;  
\item Game information: what information at which state should be available to whom according to the game.  
\end{itemize}
\item Player assumptions:
\begin{itemize}
\item Initial uncertainties: uncertainty about where to start;
\item Personal observability: how players see each others' actions;
\item Knowledge updating mechanism: how players handle new information.
\end{itemize}
\end{itemize}
Therefore a complete game model for generating the game runs is a tuple:  $$\lr{\text{game arena, game infomation, initial uncertainties, observation of actions, update mechanism}}$$

In this way, there is no double meaning for information sets any more: they are either the information provided by the games or the information that players learn during the plays. For example, in card games, each state of the game (based on the distributions of the cards) should at least give each player the information of his own hand, whose turn it is, and the previous cards that others played in the current round. As for the player assumptions in a card game, the players should not be  sure about the initial distribution of the game since the cards are usually drawn randomly from a deck. During the play, players may see each others' actions unless they do not pay attention to. However, if the game involves privately getting new cards from the remaining deck, the players should not be able to distinguish others' actions of getting different cards unless they cheat. Finally, the knowledge updating mechanism determines the reasoning power of the players and implicitly captures the memory of the agents. For example, although a player has exactly the same information about two states sharing the same hand of hers (she cannot see others' hands), she may still be able to distinguish those two states in the real play. For example, if Spade 3 has been played before, then a perfect recall player can exclude the state where the opponent has Spade 3 from then on, although the game rules give her exactly the same information on the state with or without Spade 3 at the opponent's hand. 


\medskip

Now, how do we compute the game runs from game rules? Roughly speaking, the computation schema is as follows: 
$$\text{Game runs = Game rules $\times$ Player assumptions} $$
More precisely, we start from the initial uncertainty, and compute the next level using the update mechanism based on the game information and observability of the players, guided by the game arena about the available actions. The idea comes from the connection between epistemic temporal logic (\ETL) and dynamic epistemic logic (\DEL) explored by \cite{MergingJournal09}, where iteratedly updating the epistemic model produces a temporal structure with epistemic relations. The exact computation is also inspired by the update product proposed in dynamic epistemic logic by \cite{BMS} where perfect recall is assumed implicitly. In general, our computation of the game runs can be viewed as an attempt towards the theory of play proposed by \cite{LG14} where what the players actually do and learn during the plays are crucial. 

\medskip

In this paper, based on the above ideas, we develop a dynamic epistemic logic to describe the properties of players during the plays. Interestingly, the logic stays the same if we define a static semantics over a special class of two-dimensional ETL models which captures the essential properties of the computed runs. We will use this connection to axiomatize our dynamic epistemic logic completely.

Another contribution of the paper is that we give some conditions under which the game rules themselves are just enough to induce the game runs modulo some structure equivalence notions. This may explain why in some cases the distinction between the game rule view and the game run view is blurred. 

\subsection{Related work}

\begin{itemize}
\item  \cite{Halpern97Ab} sharply points out the ambiguity of the interpretations of game trees with information sets, and uses interpreted systems to model the game plays in order to solve the puzzles about imperfect recall (cf. also  \cite{Halpern03}). \cite{HalpernM07} use this framework to elegantly unify solution concepts of extensive form games by very simple knowledge-based programs. However, this line of work does not distinguish the rules of the game and the player assumptions as we do. Moreover, the run models there  (interpreted systems) are global descriptions which are not computed step by step according to player assumptions. 
 \item  Based on \cite{MergingJournal09}, \cite[Chapter 9]{LG14} discusses in length on how to decorate a game tree (without uncertainty) with the information sets by DEL-like updates. Characterization results are given on generated runs modulo isomorphism. Here, we separate the uncertainty given by the game rules and information obtained by the agents based on their ability, and consider the class of run models generated by rules modulo \textit{logical equivalence} rather than isomorphism, which leaves the run models much more flexible. We do not use propositional preconditions but the game rules to guide the computation of the runs. 
\item A closely related work in the setting of general game playing is \textit{Game Description Language with Imperfect information} (GDL-II) proposed by \cite{SchiffelT14}, where game rules and the assumptions about the agents are described syntactically. Our approach is largely semantic and makes a clear distinction between rules and player assumptions. We will come back to the comparison with GDL-II at the end of the paper. 
\item  \cite{GH10} introduce \textit{concurrent informational game structures} (CIGS) with two uncertainty relations: one \textit{a priori} and one \textit{empirical} capturing the uncertainty player has at the beginning of the game (like our game information) and in the actual plays respectively. In  \cite{VGEP2014}, some general thoughts are given on the computation of the empirical uncertainty based on the a priori information of the game and player abilities. 
\item Our treatment of game information based on \textit{information pieces} came from the work by \cite{Kaneko08} on inductive game theory where the game structure is not known in advance but could be explored by the players in  different plays. \cite{Kaneko08} also treat the memory of players in terms of the past experiences (sequences of information pieces) explicitly which we do not pursue here.  
\item We make use of the general techniques developed by \cite{WC12} and \cite{WA13} to axiomatize the dynamic epistemic logic without reduction axioms, via a class of \textit{normal} run models. Note that there is no action preconditions but the game rules to guide the executions of the actions, which bring us closer to the work on dynamic epistemic logics over transition systems explored in \cite{WL12,YLW15}. 
\end{itemize}

\paragraph{Structure of the paper} Section~\ref{sec.rules} fleshes out formally the above informal ideas of the game rules and player assumptions. Section~\ref{sec.runs} shows how to compute the runs, which satisfy certain normality conditions. We also show that, under certain conditions about the players, rules and the computed runs are structurally similar. Section~\ref{sec.lansem}  introduces a simple epistemic language to talk about the properties of the game, with \emph{two} semantics based on rules (DEL-like) and normal runs (ETL-like) respectively. It is proven that the two semantics are equivalent in terms of the validities. In Section~\ref{sec.axiom} we give a complete axiomatization of our logic to clarify the underlying assumptions of our framework from a syntactic perspective. In Section~\ref{sec.con} we conclude with future directions. 

\medskip

\section{Game rules and player assumptions}
\label{sec.rules}

In this section we introduce a formal account of the game rules, initial uncertainties, personal observability, and the knowledge updating mechanism. Let us begin with the formalization of our intuitive ideas of the game rules. 

\begin{definition}[Game structure]\label{Df.Game}
Given a non-empty finite set $\Ag$ of players, a non-empty finite set of actions $\Act=\uplus_{i\in\Ag}A_{i}$ consisting of mutually disjoint sets $A_i$ for each player $i$, a finite set of information pieces $\BP$ such that $\{p^a\mid a\in \Act\}\subseteq \BP$, a \emph{game structure} w.r.t.\ $\Act$, $\Ag$, and $\BP$ is a tuple:  $\G=\lr{S,R,O}$ where:
\begin{itemize}
\item $S$ is a non-empty set of states;
\item $R$ is a partial function $S\times \Act \to S$ s.t. if $R(s,a)$ is defined, then $\iota(s)=i$ for some $i\in \Ag$, where $\iota(s)=i$ iff $\emptyset\subset\{a\mid R(s,a) \text{ is defined} \}$ $\subseteq A_{i}$;
\item $O: S \times \Ag \to 2^\BP$ is a function assigning each game state $s$ and each player $i$ a set of external information pieces that are available to $i$ such that for each $a\in A_i$, $p^a\in O(s,i)$ iff $R(s,a)$ is defined.
\end{itemize} 
$\lr{S,R}$ is called the \emph{game arena} and $O$ gives the \emph{game information}, and $\iota$ is the \emph{player assignment function} induced by $R$. 
To simplify notation, we write $sR_at$ for $R(s,a)=t$. 
\end{definition}

Given a game structure $\G$, we denote its set of states, relations and information by $S_\G$, $R^\G_{a}$ and $O_\G$. Given an $s\in S_\G$, let $e(s)$ be the set of available actions at $s$, i.e.,\ $\{a\mid R(s,a) \textrm{ is defined}\}$. Let $e_i(s)=e(s)$ if $\iota(s)=i$ otherwise $e_i(s)=\emptyset$.  
\medskip

Note that in this paper we separate the actions by different players in a syntactic way, which helps to simplify the presentation. We can read the functions $R$ and $O$ as collections of conditionals which fit the intuitive idea of game `rules': if player $i$ is at state $s$ then certain actions and information are available to $i$. By the definition of $R$, the available actions at a game state are deterministic, and they all (if any) belong to a single player. Moreover, we do not require the game structure to be a tree-like structure and cycles are allowed. 

Intuitively, the information pieces can be viewed as basic propositions and $O$ is then a player-dependent valuation about those propositions. In particular, $\{p^a\mid a\in\Act\}$ tells the agents what actions are available. Note that, unlike \cite{Kaneko08},  different players can be given different information on the same state, e.g., the players who are not to make the choices may not get the information on the available actions according to our condition on $O$. Clearly, the game information induces an equivalence relation and thus a partition over the game arena for each agent, which captures the rule-view of information sets.\footnote{A crucial difference between game information and information that player may learn later in the runs is that the equivalence relation of the game information should be common knowledge, while during the play not everything is commonly known intuitively.}

\medskip

Sometimes the players may not be sure about where they start exactly in the game, e.g., a random distribution of cards may let the players wonder about each others' hands. The following definition describes this (initial) uncertainty using Kripke model like structures w.r.t.\ a given game.

\begin{definition}[Epistemic model]\label{Df.epi}
An \textit{epistemic model} w.r.t.\ a game structure $\G$ is a tuple: $\M=(W, f, \sim, V)$ where:
\begin{itemize}
\item  $W$ is a nonempty set;
\item $f:W\to S_\G$ assigns each $w\in W$ a state in the game $\G$;
\item  $\sim: \Ag\to 2^{W\times W}$ assigns an equivalence relation over $W$ to each $i\in\Ag$ such that $w\sim_i v$ implies $O_\G(f(w),i)=O_\G(f(v),i)$.  
\item $V:W\times \Ag\to 2^\BP$ such that $V(w,i)=O_\G(f(w),i).$  
\end{itemize}
A pointed epistemic model is an epistemic model with a designated world. When $\G$ and $\BP$ are clear we use $\M$ or $(W, f, \sim, V)$  to denote an epistemic model.  
\end{definition}

Note that $\sim$ is restricted by $\O_\G$, for after all if game rules tell a player that he/she is at state $s$ instead of state $t$, then surely he/she knows. Due to our previous assumption about $\O_\G$ the players always know what they can do when they are supposed to move. Formally we can verify that $w\sim_i v$ implies $e_i(f(w))=e_i(f(v))$ for any $w, v$ in $\M$.

\medskip

In the actual plays, players may acquire more information than those provided by game rules. This is because players may be able to observe more (e.g., peeking others' hands in a card game). In this paper we use the following structure to capture the observation power in terms of indistinguishability among actions. 

\begin{definition}[Observation model]\label{Def.Upd}
An \textit{observation model} $\U$ w.r.t.\ \Act\ is a tuple $(\Act, \tot)$ where:
\begin{itemize} 
\item ${\tot}:\Ag\to 2^{\Act\times \Act}$ assigns an equivalence relation over $\Act$ to each player. 
\end{itemize}
\end{definition}

Intuitively, observation models characterize players' observation power for actions being conducted. Under this interpretation, $\Ue\tot_i \Uf$ can be read as, given actions $\Ue$ and $\Uf$, player $i$ is unable to distinguish which one has been conducted. The equivalence relations clearly induce partitions over $\Act$ for each $i$.
\footnote{Note that in our definition of $\U$, we may also assume that players can always distinguish actions conducted by themselves. However we do not make this restriction in this paper. This is not only for simplicity, but also because, for instance, players may be required to act blindfolded in certain games. Furthermore, since if every player cannot distinguish $\Ue$ and $\Uf$, then $\Ue$ and $\Uf$ can be treated as the same action with two possible outcomes, which may also allow us to describe a certain degree of indeterminacy, even though we presume determinacy in game rules.} 
 

\medskip

Let $\MM$ be the class of all the epistemic models w.r.t.\ a certain $\G$.  Given an observation model $\U$, an \emph{updater} $\circledast$ w.r.t. $\U$ is a function:  $\MM\to \MM$ such that $W_{\circledast(\M)}\subseteq W_{\M}\times \Act_\U$. Following the convention, we write $\M\cic\U$ for $\cic (\M)$ and write $\M\cic\U^k$ for $\M\cic\underbrace{\U \dots \cic\U}_k$, with $\M\cic\U^0=\M$.\\


The runs of an imperfect information game is determined by $\lr{\G, W, f, \sim, w, \U, \circledast}$. In this paper we focus on the following updater inspired by \cite{BMS}:  

\begin{defn}{Update Product $\otimes$}\index{product update}\label{df:execution}
Fixing $\G$ and thus $\Act$, an epistemic model $\M=(W, f, \sim, V)$ and an observation model $\U=(\Act,\tot)$, the product model is an epistemic model $(\M\otimes \U)=(W', f', \sim', V')$ where:\\
$\begin{array}{rcl}
W'&=& \{(w,a)\mid w\in W, a\in e(f(w))\}\\
f'((w,a))&=&R(f(w),a)\\
\sim_i'&=& \{((w,a),(u,b))\mid w\sim_i u, a\tot_i b \text{ and }\\
&&  O_\G(R(f(w),a),i)=O_\G(R(f(u),b),i)\}\\
V'((w,a),i)&=&O_\G(R(f(w),a),i) 
\end{array}$  
\end{defn} 

It is worth noting that we assume that players can perfectly recall but have limited reasoning power, that is, players' knowledge relies only on previous knowledge, personal observability and information provided by game rules. Also note that the above definition can make sure that $\M\otimes \U$ is still a well-defined epistemic model w.r.t.\ \G\ unless $W'=\emptyset$, as is shown in the following proposition. 
 
\begin{proposition}\label{Prop.UpdEpM}
Let $\G$ be any game structure, and let $\M$ be any epistemic model w.r.t. \G. Then $\M\otimes\U$ is still an epistemic model w.r.t.\ $\G$ if its domain is not empty.
\end{proposition}

\noindent\begin{proof}
Let $\M=(W,f,\sim,V)$ and $\M\otimes\U=(W',f',\sim',V')$. Suppose $W'$ is not empty. Then it is easy to check, by Definition $\ref{df:execution}$ that $f'$ is a function from $W'$ to $S_{\G}$. $\sim_i$ is an equivalence relation because both $\sim_i$ and $\tot_i$ are equivalence relation.

Let $i\in \Ag$ and $(w,a),(u,b)\in W'$. First we show that 
$$(w,a)\sim_i'(u,b) \textrm{ implies } O_\G(f'(w,a),i)=O_\G(f'(u,b),i).$$ 
Suppose $(w,a)\sim_i'(u,b)$. Then $O_\G(R(f(w),a),i)=O_\G(R(f(u),b),i)$. Because $f'(w,a)=R(f(w),a)$ and $f'(u,b)=R(f(u),b)$, we obtain  $O_\G(f'(w,a),i)=O_\G(f'(u,b),i)$. 

Then we need to prove that $$V'((w,a),i)=O_\G(f'(w,a),i).$$
But this is straightforward since by Definition \ref{Def.Upd}, $f'(w,a)=R(f(w),a)$ and $V'((w,a),i)=O_\G(R(f(w),a),i)$.
\end{proof}
\section{Game runs}
\label{sec.runs}

In this section, we first define a class of epistemic temporal models with both the epistemic relations and the transitions representing actions. We use such structures to represent game runs which are computed from game rules and player assumptions. These run models are constructed, layer by layer, by the update mechanism in Definition \ref{df:execution}, such that the founding layer is the epistemic model representing the initial uncertainty, and each other layer is an updated model of the previous layer. Furthermore, between successive layers there are action relations linking states in one layer to the successor states in the next layer. We show that these run models are epistemic temporal models with certain properties.

In the second subsection, we compare game structures and run models. First we establish a result showing that it is unnecessary to distinguish run models between game rules for perfect information games, modulo p-morphism. Second, we analyze how observation power affects players' knowledge provided that the game has no initial uncertainty and everyone can perfectly recall.

Note that in this section and the following sections we will fix the players set $\Ag$, the action set $\Act$ and the set of information pieces $\BP$.

\subsection{Computing the runs}

We first introduce a 2-dimensional model incorporating both knowledge and actions to accommodate the runs of the games with uncertainty.
\begin{defn}{Epistemic temporal model}
An \emph{epistemic temporal model} $\N$ is a tuple: 
$$
(W, \{\sim_{i}\mid i\in \Ag\}, \{\rel{\Ue}\mid a\in \Act\}, V)
$$
where:
\begin{itemize} 
\item $W$ is a nonempty set;
\item for each $i\in \Ag$, $\sim _{i}$ is an equivalence relation over $W$;
\item for each $\Ue\in \Act$, ${\rel{\Ue}}\subseteq W\times W$;
\item $V$ is a function: $W\times \Ag \to 2^\BP$.
\end{itemize}
\noindent Since $\Ag$ and $\Act$ are fixed, we will write $(W, \sim, \to, V)$ for $(W, \{\sim_{i}\mid i\in \Ag\}, \{\rel{\Ue}\mid \Ue\in \Act\}, V)$. 
Furthermore, we call $(W, \sim, \to)$ the \emph{epistemic temporal frame} of $\N=(W, \sim, \to, V)$, and $(W, \sim, V)$ the \emph{epistemic core} of $\N$ (notation $\N^-$). 
For a sequence of events $h=a_0\dots a_k$ we write $w\rel{h}w'$ for $w\rel{a_0}\dots\rel{a_k}w'$.
\end{defn}

We use $Act(s)$ for $\{a\mid \exists t\in W \textrm{ such that } s\rel{a}t \}$, and $Turn(s)$ for $\{i\mid Act(s)\neq \emptyset \textrm{ and }Act(s)\subseteq A_{i} \}$. Recall that in a game structure, $e(s)=\{a\mid R(s,a) \textrm{ is defined}\}$, and $\iota(s)=i$ iff $\emptyset \subset e(s) \subseteq A_i$. $Act(s)$ and $Turn(s)$ are essentially the counterparts of $e(s)$ and $\iota(s)$ in epistemic temporal models. Note that $Turn(s)$ can be empty when $s$ enables actions from different players.  



\medskip

Now we define the run model generated by updating $\U$ repeatedly given a game structure $\G$ using the ideas in \cite{MergingJournal09}:

\begin{defn}{Run model}\label{Df.UpForest}
Given an epistemic model $\M$ w.r.t.\ a game structure $\G$. The \emph{run model generated by repeatedly updating $\U$}, written as $\F^\otimes(\M)$, is an epistemic temporal model: $$(W,\{\sim_i\mid i\in \Ag\},\{\rel{\Ue}\mid \Ue\in \Act\}, V)$$ where: 
\begin{itemize} 
\item $W=\{(s, a_1,\dots, a_k) \mid$ $0\leq k$, $(s,a_1,\dots,a_k)$ exists in $\M\otimes\U^k$ and $f_{\M}(s)R^{\G}_{a_1}t_1R^{\G}_{a_2}\dots R^{\G}_{a_k}t_k  \textrm{ for }   \textrm{some } t_1,\dots,t_k\in S_\G \}$;
\item $(s, a_1,\dots, a_k)\sim_i (t, a'_1,\dots, a'_j)$ iff 
$k=j$, and $(s,a_1,\dots,a_k)\sim_i (t,a'_1,\dots,a'_k)$ in $\M\cic \U^k$;
\item $(s, a_1,\dots, a_k)\rel{\Ue} (t, a'_1,\dots, a'_j)$ iff 
 $s=t$ and  $a'_1\dots a'_j=a_1\dots a_k a$;
\item $V((s, a_1,\dots, a_k), i)=V_{\M\otimes\U^k}((s,a_1,\dots,a_k), i)$.
\end{itemize}
The induced state assignment function $f:W\to S_\G$ w.r.t.\ $\F^\otimes(\M)$ is given by:
\begin{itemize}
\item $f((s, a_1,\dots, a_k))=f_{\M\otimes\U^k}((s, a_1,\dots, a_k))$. 
\end{itemize}
\end{defn}
Intuitively, $\F^\otimes(\M)$ is an \textit{update universe} consisting of the updated models in the form of $\M\otimes\U^k$ linked by update transitions $\rel{a}$. Note that we assume perfect recall, and our definition of $\sim_i$ represents this assumption.

\begin{remark}\label{Rem.Run}
Given an observation model $\U$ and an epistemic model $\M$ w.r.t.\ a game $\G$, 
the run model $\F^\otimes(\M)=(W,\sim,\to,V)$ and its induced state assignment function $f$ satisfy  that for all $w,u\in W$, $\Ue,\Uf\in \Act$ and $i\in\Ag$:
\begin{itemize}
 \item $f((w,a))=R_\G(f(w),a);$
 \item if $a\in e(f(w))$, then $(w,a)\in W$;
 \item if there are $s,t\in W$ such that $s\rel{\Ue} w$ and $t\rel{\Uf} u$, then:\\ $w\sim_i u$ iff $s\sim_i t$, $\Ue\tot_i \Uf$ and $O_\G(f(w),i)=O_\G(f(u),i);$
 \item $V(w,i)=O_\G(f(w),i).$
\end{itemize}  
\end{remark} 

 

We give some examples below to illustrate our ideas of the computation of game runs from game rules and player assumptions (we omit the game information $O$ and only show the induced information sets for player 2 in the game structures; singleton information sets are omitted, similarly for the reflexive epistemic arrows in the run models):\\

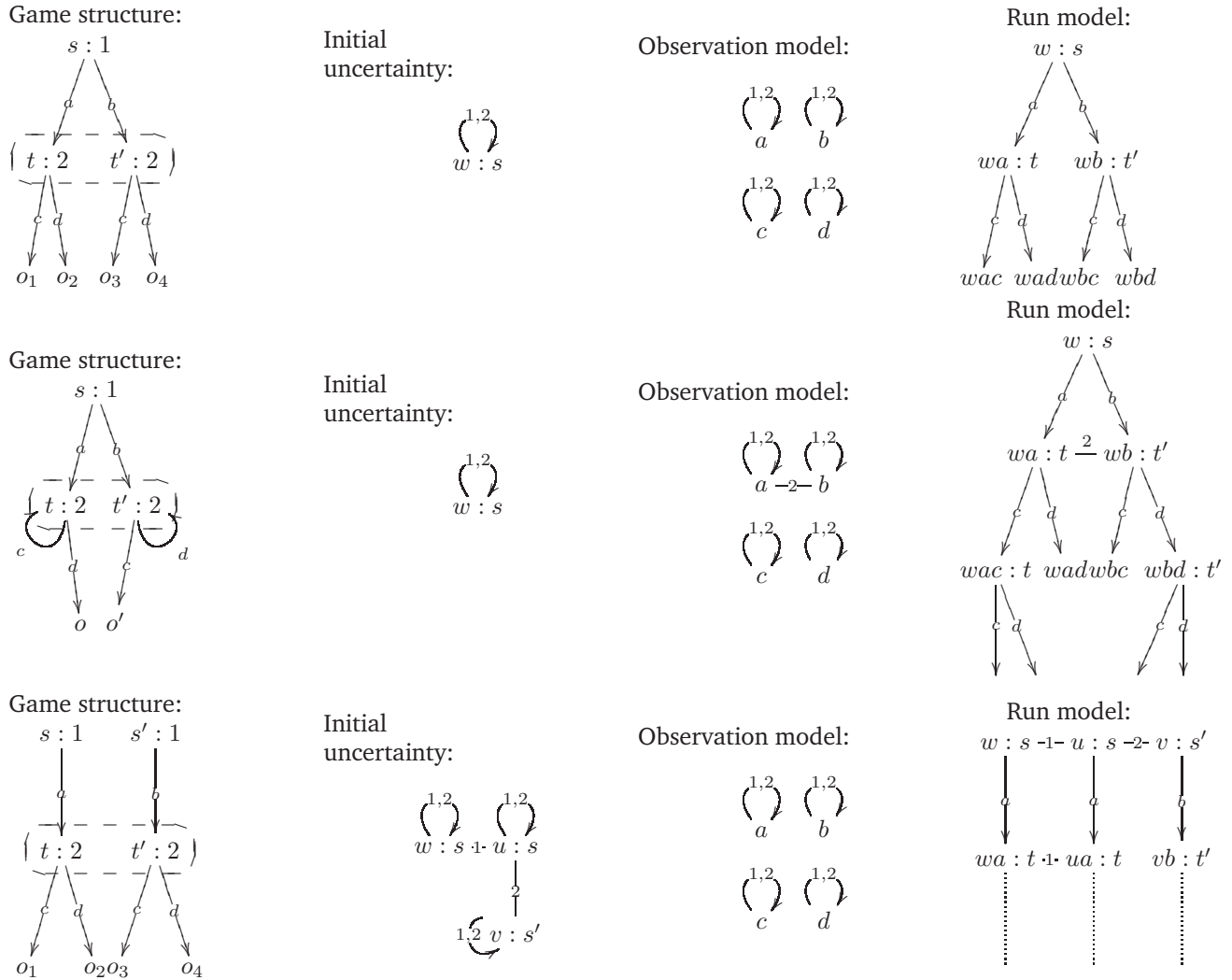
\begin{minipage}{0.25\textwidth}
\small{Game structure:}\\
{$\small
\xymatrix@R+10pt@C-35pt{
&&&{s:1}\ar[dll]|a\ar[drr]|b\\
&t:2\ar[dl]|c\ar[dr]|d&&&&t':2\ar[dl]|c\ar[dr]|d\\
o_1& &o_2&&o_3&&o_4
\save "2,2"."2,6"!C="g1"*+[F--:<+20pt>]\frm{}
\restore
}
$
}
\end{minipage}
\begin{minipage}{0.25\textwidth}
\small{Initial \\uncertainty:\\}
{$$\small
\xymatrix@C-20pt{
w:s\ar@(ul,ur)|{1,2}\\
{}
}
$$}
\end{minipage}
\begin{minipage}{0.25\textwidth}
\small{Observation model:}\\
{$$\small
\xymatrix@C-10pt{
{a}\ar@(ul,ur)|{1,2}&{b}\ar@(ul,ur)|{1,2}\\
{c}\ar@(ul,ur)|{1,2}&{d}\ar@(ul,ur)|{1,2}
}$$
}
\end{minipage}
\begin{minipage}{0.25\textwidth}
\qquad\small{Run model:}\\
{$\small
\xymatrix@R+10pt@C-40pt{
&&&w:s\ar[dll]|a\ar[drr]|b\\
&{wa:t}\ar[dl]|c\ar[dr]|d&&&&{wb:t'}\ar[dl]|c\ar[dr]|d\\
wac& &wad&&wbc&&wbd
}
$}
\end{minipage}
\\

\begin{minipage}{0.25\textwidth}
\small{Game structure:}\\
{$\small
\xymatrix@R+10pt@C-35pt{
&&&{s:1}\ar[dll]|a\ar[drr]|b\\
&t:2\ar@(d,l)^c\ar[dr]|d&&&&t':2\ar[dl]|c\ar@(d,r)_d\\
& &o&&o'&&
\save "2,2"."2,6"!C="g1"*+[F--:<+10pt>]\frm{}
\restore
}
$
}
\end{minipage}
\begin{minipage}{0.25\textwidth}
\small{Initial \\uncertainty:\\}
{$$\small
\xymatrix@C-20pt{
{w:s}\ar@(ul,ur)|{1,2}\\
{}
}
$$}
\end{minipage}
\begin{minipage}{0.25\textwidth}
\small{Observation model:}\\
{$$\small
\xymatrix@C-10pt{
{a}\ar@(ul,ur)|{1,2}\ar@{-}[r]|2&{b}\ar@(ul,ur)|{1,2}\\
{c}\ar@(ul,ur)|{1,2}&{d}\ar@(ul,ur)|{1,2}
}$$
}
\end{minipage}
\begin{minipage}{0.25\textwidth}
\qquad\small{Run model:}\\
{$\small
\xymatrix@R+10pt@C-40pt{
&&&w:s\ar[dll]|a\ar[drr]|b\\
&{wa:t}\ar@{-}[rrrr]^2\ar[dl]|c\ar[dr]|d&&&&wb:t'\ar[dl]|c\ar[dr]|d\\
{wac:t}\ar[d]|c\ar[dr]|d& &wad&&wbc&&{wbd:t'}\ar[d]|d\ar[dl]|c\\
&&&&&&}
$}
\end{minipage}\\

\begin{minipage}{0.25\textwidth}
\small{Game structure:}\\
{$\small
\xymatrix@R+10pt@C-30pt{
&s:1\ar[d]|a&&&&s':1\ar[d]|b\\
&t:2\ar[dl]|c\ar[dr]|d&&&&t':2\ar[dl]|c\ar[dr]|d\\
o_1& &o_2&&o_3&&o_4
\save "2,2"."2,6"!C="g1"*+[F--:<+30pt>]\frm{}
\restore
}
$
}
\end{minipage}
\begin{minipage}{0.25\textwidth}
\small{Initial \\uncertainty:\\}
{$$\small
\xymatrix@C-20pt{
w:s\ar@{-}[r]|1\ar@(ul,ur)|{1,2}&{u:s}\ar@(ul,ur)|{1,2}\ar@{-}[d]|2\\
&{v:s'}\ar@(lu,ld)|{1,2}
}
$$}
\end{minipage}
\begin{minipage}{0.25\textwidth}
\small{Observation model:}\\
{$$\small
\xymatrix@C-10pt{
{a}\ar@(ul,ur)|{1,2}&{b}\ar@(ul,ur)|{1,2}\\
{c}\ar@(ul,ur)|{1,2}&{d}\ar@(ul,ur)|{1,2}
}$$
}
\end{minipage}
\begin{minipage}{0.25\textwidth}
\qquad\small{Run model:}\\
{$\small
\xymatrix@R+10pt@C-24pt{
&w:s\ar[d]|a\ar@{-}[rr]|1&&{u:s}\ar[d]|a\ar@{-}[rr]|2&&{v:s'}\ar[d]|b\\
&{wa:t}\ar@{.}[d]\ar@{-}[rr]|1&&{ua:t}\ar@{.}[d]&&{vb:t'}\ar@{.}[d]\\
&&&&&
}
$}
\end{minipage}\\

In the following, we define a class of epistemic temporal models with respect to a given observation model $\U$, to be used in the later proofs.
\begin{definition}[Normal epistemic temporal model]
An epistemic temporal model $\N=(W,\sim,\to, V)$ is \emph{normal} w.r.t.\ $\U$ if the following properties hold for all $s,t,t'$ in $\N$, $a,b\in \Act$ and $i$ in $I$:
\begin{description} 
\item[\NMp\ (no miracles)]  if $s\sim_i s'$ and $s'\rel{\Ue}t'$ then for all $\Uf$ and $t$ such that $s\rel{\Uf}t$, $b\tot_i a$ and $V(t,i)=V(t',i)$, we have $t\sim_i t'$;
\item[\PRp\ (perfect recall)] if $s\rel{\Ue}t$ and $t\sim_i t'$ then there is an $s'$ such that $s\sim_i s'$ and $s'\rel{\Uf} t'$ for some $b$ such that $a\tot_i b$ in $\U$;
\item [(Det)] if $s\rel{\Ue}t$ and $s\rel{\Ue}t'$, then $t=t'$;
\item[(Exturn)] if $Act(s)\neq \emptyset$, then $Act(s)\subseteq A_{j}$ for some $j\in I$;
\item[(Info)] for each $a\in A_i$, $p^a\in V(s,i)$ iff $s\rel{\Ue}t'$ for some $t'\in W$;

\item[(Ke)] if $w\sim_iv$ then $V(w,i)=V(v,i)$.
\end{description}

\end{definition}

Now we show that all the generated models are normal. 
\begin{proposition}\label{Prop.RunisNem}
Given an epistemic model $\M$ w.r.t.\ a game structure $\G$, let $\F^\otimes(\M)=(W,\sim,\to, V)$ be the epistemic temporal model generated by repeatedly updating $\U$, and let $f$ be the induced state assignment function w.r.t.\ $\F^\otimes(\M)$. Then $\F^\otimes(\M)$ is a normal epistemic temporal model.
\end{proposition}

\noindent\begin{proof} We check that $\F^\otimes(\M)$ satisfies all the conditions of normal epistemic temporal model. Let $$w=(w',a_1,\ldots,a_k),u=(u',b_1,\ldots,b_n),v=(v',c_1,\ldots,c_m)\in W,$$ $\Ue\in \Ag$ and $i\in\Ag$.

$\mathbf{Nm}$. Suppose $w\sim_i u$ and $u\rel{\Ue} v$. Consider any $s\in W$ and $\Uf\in\Act$ such that $w\rel{\Uf} s$, $\Ue\tot_i \Uf$ and $V(v,i)=V(s,i)$. Using Remark \ref{Rem.Run} we have $v\sim_i s$.

$\mathbf{Pr}$. Suppose $w\rel{\Ue} u$ and $u\sim_i v$. Then $n=m$, and hence $(v',c_1,\ldots,c_{m-1})\rel{c_m}v$. Using Definition \ref{Df.UpForest} we have $w\sim_i (v',c_1,\ldots,c_{m-1})$ and $\Ue\tot_i c_m$

$\mathbf{Det}$. By the definition of $W$ and $\to$, it is easy to see.

$\mathbf{Exturn}$. Suppose $\Ue\in Act(w)$. Then by Remark \ref{Rem.Run}, $f((w,a))=R_\G (f(w),a)$. It follows that, by Definition \ref{Df.Game}, $e(f(w))\subseteq A_j$ for some $j\in \Ag$, which implies, using Definition \ref{Df.UpForest}, $Act(w)\subseteq A_j$.

$\mathbf{Info}$. Suppose $a\in A_i$. Then $p^a\in V(w,i)$ iff $p^a\in O_\G(f(w),i)$ iff $R_\G (f(w),a)$ is defined, iff $w\rel{\Ue} s$ for some $s\in W$.

$\mathbf{Ke}$. Suppose $w\sim_i u$. By Remark \ref{Rem.Run} we have $V(w,i)=O_\G(f(w),i)=O_\G(f(u),i)=V(u,i)$.
\end{proof}

Later in the paper we will show that any normal epistemic temporal model can also be generated from some epistemic model modulo logical equivalence w.r.t.\ the language to be introduced later.  

\subsection{Structural comparison between game rules and run models}

In this subsection, we first use p-morphism to demonstrate that for perfect information games, the game structures are actually the  p-morphic images of their run models, no matter what observation model is assumed. Here a \textit{perfect information game} is a game structure whose induced information sets for each player are singletons, i.e., each game state is given a unique set of information pieces for each player. Next, we examine the impact of observational power on players' knowledge given perfect recall w.r.t.\ the game information and certainty of the initial state.

\medskip
The following definition of p-morphism is an adaptation of the standard one in modal logic. 

\begin{defn}{p-morphism}
Let $\M=(W, \{\sim_i\mid i\in \Ag\},\{R_a\mid a\in \Act\}, V)$ and $\N=(W',\{\sim'_i\mid i\in \Ag\}, \{R'_a\mid a\in \Act\},V')$ be two epistemic temporal models. A function $f:W\to W'$ is a \emph{p-morphism} from $\M$ to $\N$ if for all $w,u\in W$, $s\in W'$ and $a\in \Act$:
 \begin{itemize}
   \item for each $i\in \Ag$, $V(w,i)=V'(f(w),i)$;
  \item $w\sim_i u$ only if $f(w)\sim'_if(u)$;
\item if $f(w)\sim'_i s$, then there is a $v\in W$ such that $f(v)=s$ and $w\sim_i v$;
  \item $wR_a u$ only if $f(w)R'_af(u)$;
  \item if $f(w)R'_a s$, then there is a $v\in W$ such that $f(v)=s$ and $wR_a v$.
 \end{itemize}
\end{defn}
It is a standard result in modal logic that p-morphisms preserve the truth value of modal formulas (cf. \cite{mlbook}). 

\medskip

To facilitate the comparison, we can view a game structure as an  epistemic temporal model where the indistinguishability relation is induced by the information pieces.

\begin{defn}{Induced epistemic game structure}\label{Def.epiGame}
Given a game structure $\G=(S,R,O)$, the induced epistemic game structure $E(\G)$ is $(S,\sim, R, O)$ where $\sim_i=\{(s,t)\mid O(s,i)=O(t,i)\}$. A perfect information game is a game structure where  $\sim_i$ are identity relations in $E(\G)$.
\end{defn}
Note that $E(\G)$ can be viewed as an epistemic temporal structure. The following proposition shows that for perfect information games, the induced state assignment functions are p-morphisms from their run models to the induced epistemic game structures.
Also note that this also means that the state assignment functions are also p-morphisms from the run models to the game structures,  if we neglect the epistemic relations in the run models and the conditions for $\sim$ in p-morphisms. 

\begin{proposition}
Let $\G$ be a perfect information game, $\M$ be an epistemic model, $\F^{\otimes}(\M)=(W,\sim,\to, V)$ be the run model generated by repeatedly updating some observation model $\U$, and $f$ be the induced state assignment function w.r.t.\ $\F^{\otimes}(\M)$. Then $f$ is a p-morphism from $\F^{\otimes}(\M)$ to $E(\G)=(S, \sim', R,  O)$.
\end{proposition}

\noindent\begin{proof}
Note that by Definition \ref{Def.epiGame}, $\G=(S,R,O)$. Let $w,u\in W$, $s\in S$ and $\Ue\in\Act$. First we check the conditions for the action transitions. Clearly $w\rel{\Ue}u$ implies $f(w)R_\Ue f(u)$ by the definition of the run model. On the other hand, suppose $f(w)R_\Ue s$ then again by the definition of the run models there is a $u=(w,a)$ in $W$ such that $f(u)=f(w,a)=R(f(w),a)=s$ by Remark \ref{Rem.Run} and the fact that $R$ is a partial function. By definition  $w\rel{\Ue}u$ in the run model. Second, we need to verify the conditions for the epistemic relations. Suppose $w\sim_iu$, then $O(f(w),i)=O(f(u),i)$ in $\G$ by the definition of epistemic relations in the run models. Now $f(w)\sim_if(u)$ in $E(\G)$ by the definition. On the other hand, suppose that $f(w)\sim'_is$ in $E(\G)$. Since $\G$ is a perfect information game $f(w)=s$. It is clear that $w\sim_iw$ in the run model.  Finally, by Remark \ref{Rem.Run}, for each $i\in \Ag$, $V(w,i)=O(f(w),i)$, and this completes our proof.
\end{proof}

Note that, in the above proof, the observation model $\U$ does not play a role under the assumption of perfect information. The above result may explain why we do not need to distinguish the run view from the rule view in perfect information games. However, for imperfect information games in general, the induced epistemic temporal models are very different from the run models which crucially depend on the observation models. There is no general structural relation like p-morphism or bisimulation between them. In the following, we unravel the game structures into game trees and induce epistemic relations on them w.r.t.\ (a version of) perfect recall to explore the similarity between game structures and run models further. In this way, we can also exam the impact of observation models precisely.

\begin{definition}[Generated epistemic game tree]\label{Df.GameTree}
Let $\G=(S,R,O)$ be a game structure, and let $s\in S$. The \emph{epistemic game tree $\G^{\upharpoonright s}$ generated from }$s$ with perfect recall is a tuple $(S',\{\frown_i\mid i\in \Ag\}, R', O')$ where
 \begin{itemize}
  \item $S'=\{(s, a_1,\dots, a_n) \mid 0\leq n \textmd{ and } sR_{a_1}t_1R_{a_2}\cdots R_{a_n}t_n \textmd{ for some } a_1,\ldots,a_n\in\Act \textrm{ and } t_1,\ldots,t_n\in S\}$;
    \item for all $w,u\in S'$ and $i\in \Ag$, $w\frown_i u$ iff there are $a_1,\ldots,a_n,b_1,\ldots,b_n\in\Act$ and $w_1,\ldots,w_n,u_1,\ldots,u_n \in S'$ s.t. $w=w_n$, $u=u_n$, $sR'_{a_1}w_1R'_{a_2}\cdots R'_{a_n}w_n$, $sR'_{b_1}u_1R'_{b_2}\cdots R'_{b_n}u_n$ and $O'(w_k,i)=O'(u_k,i)$ for each natural number $k$ with $1\leq k\leq n$;
  \item for each $\Ue\in\Act$, $R'_\Ue=\{(w,u)\in S'\times S'\mid u=(w,a)\textrm{ for some }a\in \Act\}$;
  \item for each $t=(s, a_1,\dots, a_n)\in S'$ and each $i\in\Ag$, $O'(t,i)=O(t_n,i)$, where $t_n$ is the state such that $sR_{a_1}t_1R_{a_2}\cdots R_{a_n}t_n$ for some  $t_1,\ldots,t_{n-1}\in S$.

 \end{itemize}
\end{definition}

Note that in the above definition, $\frown_i$ is computed by synchronously matching the history of the game information given to agent $i$. It is a version of prefect recall without considering the agents' observation power of actions. It can be verified that if $w\frown_iu$ and $O'((w,a),i)=O'((u,b),i)$ then $(w,a)\frown_i (u,b)$. 

\medskip

In the following definition, we introduce epistemic models without initial uncertainty. We also introduce a natural correspondence between run models and epistemic game trees. Last we give a condition on observation models which characterizes the inability of observational influence on players' knowledge.

\begin{definition}
Let $\G$ be a game structure, $\M=(W,f^*,\sim,V)$ be an epistemic model w.r.t\ $\G$, $\G^{\upharpoonright s}=(S',\frown, R',O')$ be the epistemic game tree generated from some $s\in S_\G$, $\U=(\Act,\tot)$ be an observation model, $\F^{\otimes}(\M)=(W',\sim',\to, V')$ be the run model generated by repeatedly updating $\U$, and $f$ be the induced state assignment function w.r.t.\ $\F^{\otimes}(\M)$:
\begin{itemize}
 \item $\M$ is a \emph{certainty epistemic model} if $W$ is a singleton, i.e., the starting state is common knowledge;
 \item \emph{tracking function } $g:W'\to S'$ from $\F^{\otimes}(\M)$ to $\G^{\upharpoonright s}$ is a partial function given by: $g(w,a_1,\dots,a_n)=(f^*(w),a_1,\dots,a_n)$. Note that $g$ is defined only on the states $(w,a_1,\dots,a_n)$ such that $f^*(w)=s$;
  \item $\U$ is \emph{non-informative} w.r.t.\ $\F^{\otimes}(\M)$  if for all $w,u,s,t\in W'$, $\Ue,\Uf\in\Act$ and $i\in\Ag$, such that $w\rel{a} s$, $u\rel{b} t$ and $w\sim'_i u$, if $a\not\tot_i b$,  then $V'(s,i)\neq V'(t,i)$. 
\end{itemize}
\end{definition}

Intuitively the non-informative condition says that players' observability cannot provide more information than the game information. Now we prove that the non-informative condition of observation models characterizes the observational influence on knowledge of perfect recall players provided the game runs have no initial uncertainty.

\begin{theorem}
Let $\G$ be a game structure, $\U$ be an observation model, $\M=(\{v\},f^*,\sim,V)$ be a certainty epistemic model, $\G^{\upharpoonright f^*(v)}=(S', \frown, R', O')$ be the epistemic game tree generated from $f^*(v)$, $\F^{\otimes}(\M)=(W',\sim', \to, V')$ be the run model generated by repeatedly updating $\U$, and $g$ be the tracking function from $\F^{\otimes}(\M)$ to $\G^{\upharpoonright f^*(v)}$. Then
$\U$  is non-informative w.r.t.\ $\F^{\otimes}(\M)$ iff $g$  is an isomorphism from  $\F^{\otimes}(\M) \textrm{ to } \G^{\upharpoonright f^*(v)}.$
\end{theorem}
\begin{proof}
If we ignore the epistemic relations, the temporal structure of $\F^{\otimes}(\M)$ is clearly isomorphic to the unravelling of the game structure defined in $ \G^{\upharpoonright f^*(v)}$, given that $\M$ has only one state $v$. In particular, we have $V'(u,i)=O'(g(u),i)$ for any $u\in W'$. We only need to consider the epistemic relations to establish a full isomorphism between the two. 

Left to right: Assume that $\U$ is non-informative. Consider any $w=(v,a_1,\ldots,a_m),u=(v,b_1,\ldots,b_n)\in W'$ and $i\in \Ag$ we need to show that: 
$$w\sim'_iu \iff g(w)\frown_ig(u)$$
The claim holds trivially if $m\not=n$ due to the definitions of $\sim'$ and $\frown$ based on the implicitly assumed synchronicity. Thus we only need to consider the case when $m=n$. Here an inductive proof on $n$ suffices. Suppose $n=0$,  then $w=u=v$ thus both $w\sim'_iu$ and $g(w)\frown_ig(u)$ hold trivially. Suppose that the claim holds for $n=k$. Now we show the claim also holds for the case of $n=k+1$. 

If $w\sim'_iu$, then by definition of $\sim'$: $(v,a_1,\ldots,a_k)\sim'_i (v,b_1,\ldots,b_k)$ and $V'(w,i)=V'(u,i)$ thus $O'(g(w),i)=O'(g(u),i)$. By I.H., $(f^*(v),a_1,\ldots,a_k)\frown_i (f^*(v),b_1,\ldots,b_k)$ which means up to $j\leq k$ the information on $(f^*(v),a_1,\ldots,a_j)$ and $(f^*(v),b_1,\ldots,b_j)$ step-wisely coincide. Therefore by definition of $\frown_i$, $g(w)\frown_ig(u)$. 

If $g(w)\frown_ig(u)$ then $O'(g(w),i)=O'(g(u),i)$ and $(f^*(v),a_1,\ldots,a_k)\frown_i (f^*(v),b_1,\ldots,b_k)$ by the definition of $\frown_i$. By I.H. $ (v,a_1,\ldots,a_k)\sim'_i (v,b_1,\ldots,b_k)$. If $a_{k+1}\tot_i b_{k+1}$ then $w\sim'_i v$ by definition. If $a_{k+1}\not\tot_i b_{k+1}$, by non-informativeness $V'(w,i)\not=V'(v,i)$ which contradicts the fact that $O'(g(w),i)=O'(g(u),i)$. Therefore $a_{k+1}\tot_i b_{k+1}$ and then $w\sim'_i v$. 

Right to left: Suppose towards contradiction that $\U$ is not non-informative but $g$ is an isomorphism. Then there are $w,u,s,t\in W'$, $a,b\in \Act$ and $i\in\Ag$ such that $a\not\tot_i b$, $w\sim'_i u$, $w\rel{\Ue} s$, $u\rel{\Uf} t$ and $V'(s,i)=V'(t,i)$. It follows that $s\nsim_i t$ but $O'(g(s),i)=O'(g(t),i)$. Since we assume that $g$ is an isomorphism, $g(w)\frown_ig(u)$ thus by definition $g(s)\frown_i g(t)$. Now we have $s\nsim_i t$ and $ g(s)\frown_i g(t)$, contradiction. 
\end{proof}

The above simple result shows that without initial uncertainty, if action observations do not bring more distinguishing power to the players then an unraveling of the game structure can be viewed as a run model. This also explains why in some imperfect information games with perfect recall the distinction between the run view and the rule view is also blurred.

\section{Language and semantics}
In this section, we introduce a simple logic language to talk about game runs with two equivalent (modulo validity) semantics based on epistemic models and normal epistemic temporal models respectively. 
\label{sec.lansem}
\begin{definition}
Given a non-empty finite set of players $\Ag$, a non-empty finite set $\Act$ of basic actions, and a non-empty finite set \BP\ of information pieces, the \textit{dynamic epistemic language} \LDEL\ is defined as follows:
\begin{eqnarray*}
  \phi & \grammaris & \top
           \grammarsep p_i
           \grammarsep \neg \phi
           \grammarsep (\phi \land \phi)
           \grammarsep \K_i  \phi
           \grammarsep [a]\phi
\end{eqnarray*}
where $i\in\Ag$, $p\in\BP$, and $a\in\Act$. We call the $[a]$-free part of \LDEL\ the \textit{epistemic language} (\LEL). As usual, we define $\bot$, $\phi\lor \psi $, $\phi\to\psi$, $\hat{\K}_i\phi$, and
$\lr{\Ue}\phi$ as the abbreviations of $\neg\top$,
$\neg(\neg\phi\land\neg\psi)$, $\neg\phi\lor\psi$, $\neg\K_i\neg\phi$, and $\neg
[\Ue]\neg\phi$ respectively. We use $\TURN_i$ to denote $\bigvee_{a\in A_i} \lr{a}\top$.
\end{definition}

\subsection{A dynamic epistemic semantics}


Given a game structure $\G$ and an epistemic model $\M=(W, f, \sim, V)$ w.r.t.\ $\G$, the truth value of $\LDEL$ formulas at a state $s$ in $\M$ is defined as follows:\footnote{Note that we may view the model as $\M=(\G, W, f, \sim, V)$ since $f$ assumes a game $\G$. }
\begin{center}
$
\begin{array}{|rcl|}
\hline
\M,w\vDash \top  & \Leftrightarrow &   \textrm{ always }\\
\M,w\vDash p_i  & \Leftrightarrow &w \in V(p,i) \\
\M,w\vDash \neg\phi &\Leftrightarrow& \M,w\nvDash \phi \\
\M,w\vDash \phi\land \psi &\Leftrightarrow&\M,w\vDash \phi \textrm{ and } \M,w\vDash \psi \\
\M,w\vDash \K_i \psi &\Leftrightarrow&\forall v: w\sim_i v \textrm{ implies }\M,v\vDash \psi  \\
\M,w\vDash [\Ue]\phi &\Leftrightarrow& a\in e(f(w)) \textrm{ implies }
\M\otimes\U,(w,a)\vDash \phi\\
\hline
\end{array}
$
\end{center}

\begin{definition}[$\G$-bisimulation]
Given a game structure \G, a binary relation $Z$ is called a $\G$-bisimulation between two epistemic models $\M$ and $\N$ w.r.t.\ $\G$, if whenever $wZu$ the following conditions hold:
\begin{description}
\item[Inv]  $f_{\M}(w)=f_{\N}(u)$; 
\item[Zig] if $w\sim^{\M}_i w'$ for some $w'$ in $\M$ and $i\in \Ag$, then there is a $u'$ in $\N$ such that $u\sim^\N_i u'$ and $w'Zu'$;
\item[Zag] if $u\sim^\N_i u'$ for some $u'$ in $\N$ and $i\in \Ag$, then there is a $w'$ in $\M$ such that $w\sim^\M_i w'$ and $u'Zw'$.
\end{description}
We say two pointed epistemic models $\M,s$ and $\N,t$ are $\G$-bisimilar ($\M,s\bis_{\G}\N,t$) if there is a $\G$-bisimulation $Z$ between $\M$ and $\N$ such that $sZt$.
\end{definition}

Note that the invariance condition $\mathbf{Inv}$ in $\G$-bisimulation is different to that in standard bisimulation definition, where it requires $w$ and $v$ satisfy the same proposition letters. However, our (stronger)  definition implies this requirement, for $f_{\M}(w)=f_{\N}(u)$ only if for each $i\in \Ag$, $O_\G(f_{\M}(w),i)=O_\G(f_{\N}(u),i)$, and then $V_\M(w,i)=V_\N(u,i)$.

\begin{proposition}\label{Lem.TrPrsv}
Given a game structure $\G=\lr{S,R,O}$, $\LDEL$ formulas are invariant with respect to $\vDash$ under $\G$-bisimulation.
\end{proposition}

\noindent\begin{proof}
Let $\M,s$ and $\N,t$ be two pointed epistemic models that are $\G$-bisimilar, where $\M$ and $\N$ are epistemic models w.r.t.\ $\G$, and let $Z$ be the $\G$-bisimilation. The basic case for atomic $p_i$ hold because by $\mathbf{Inv}$, $f_{\M}(s)=f_{\N}(t)$, which implies that $$V_\M(s,i)=O(f_{\M}(s),i)=O(f_{\N}(t),i)=V_\N(t,i),\textrm{ for each }i\in \Ag.$$
We only consider the case $[\Ue]\phi$. 
Recall that $\M,s\vDash [\Ue]\phi$ iff $\Ue\in e(f_{\M}(s))$ implies $\M\otimes\U,(s,a)\vDash \phi$, 
and $\N,t\vDash [\Ue]\phi$ iff $\Ue\in e(f_{\N}(t))$ implies $\N\otimes\U,(t,a)\vDash \phi$.
Since $f_{\M}(s)=f_{\N}(t)$, we can obtain that $a\in e(f_{\M}(s))$ iff $a\in e(f_{\N}(t))$. If $a\notin e(f_{\M}(s))$, then $[a]\phi$ is vacuously true in both $\M,s$ and $\N,t$. Suppose $a\in e(f_{\M}(s))$. Then clearly the domains of both $\M\otimes\U$ and $\N\otimes\U$ are not empty, which implies, by Proposition \ref{Prop.UpdEpM}, that $\M\otimes\U$ and $\N\otimes\U$ are still epistemic models w.r.t.\ \G. We prove $\M\otimes\U,(s,a)\vDash \phi$ iff $\N\otimes\U,(t,a)\vDash \phi$ by using induction hypothesis and showing that $\M\otimes\U,(s,a)$ and $\N\otimes\U,(t,a)$ are $\G$-bisimilar.

Let $Z'$ be a binary relation between $\M\otimes\U$ and $\N\otimes\U$ such that for any $(w,\Uf)$ in $\M\otimes\U$ and any $(u,\Uf')$ in $\N\otimes\U$, $$(w,\Uf)Z'(u,\Uf') \textrm{ iff } wZu \textrm{ and } \Uf=\Uf'.$$ 
Clearly $(s,a)Z'(t,a)$, and then $Z'$ is nonempty.

Let $(\M\otimes \U)=(W, f, \sim, V_1)$, $(\N\otimes \U)=(W', f', \sim', V_2)$, 
let $(w,\Uf)$ in $\M\otimes\U$ and $(u,\Uf')$ in $\N\otimes\U$ such that $(w,\Uf)Z'(u,\Uf')$.
Then $\Uf=\Uf'$, and by $wZu$, $f_{\M}(w)=f_{\N}(u)$.
It follows that $$f(w,\Uf)=R(f_{\M}(w),\Uf)=R(f_{\N}(u),\Uf')=f'(u,\Uf').$$
Hence $\mathbf{Inv}$ holds for $Z'$.

For $\mathbf{Zig}$. 
Suppose $(w,\Uf)\sim_i(w',c)$ for some $(w',c)$ in $\M\otimes\U$ and some $i\in \Ag$.
Then $$w\sim_i w', \Uf\tot_i c \textmd{ and } O(R(f_{\M}(w),\Uf),i)=O(R(f_{\M}(w'),c),i).$$ 
Because $wZu$, there is a $u'$ in $\N$ such that $u\sim_i u'$ and $w'Zu'$.
By the definition of $\G$-bisimulation, $f_{\M}(w')=f_{\N}(u')$, and then $e(f_{\M}(w'))=e(f_{\N}(u'))$, which implies $c\in e(f_{\N}(u'))$. 
It follows that $R(f_\N(u'),c)$ is defined, and for each $i\in \Ag$, $$O(R(f_{\M}(w'),c),i)=O(R(f_{\N}(u'),c),i).$$
Recall that $R(f_{\M}(w),\Uf)=R(f_{\N}(u),\Uf')$, and hence $$O(R(f_{\M}(w),\Uf),i)=O(R(f_{\N}(u),\Uf'),i).$$
Thus we have 
$$ \small{ O(R(f_{\N}(u),\Uf'),i)=O(R(f_{\M}(w),\Uf),i)=O(R(f_\M(w'),b),i)=O(R(f_\N(u'),c),i)},$$
together with $u\sim_i u'$, $b'=b$ and $\Uf\tot_i c$, we obtain  $(u,\Uf')\sim_i (u',c)$. It remains to show that $(w',c)Z'(u',c)$, but it is straightforward, for $w'Zu'$ and $c=c$.

The case $\mathbf{Zag}$ is similar. Therefore $\M\otimes\U,(s,a)$ and $\N\otimes\U,(t,a)$ are $\G$-bisimilar, and this completes our proof.
\end{proof}
\subsection{An epistemic temporal semantics}
Since we also want to talk about run models which are epistemic temporal structures, we define an alternative semantics $\Vdash$ for the language \LDEL\ on epistemic temporal models:
\begin{center}
$
\begin{array}{|rcl|}
\hline
\N,w\Vdash \top  & \Leftrightarrow &   \textrm{ always }\\
\N,w\Vdash p_i  & \Leftrightarrow &w \in V(p,i) \\
\N,w\Vdash \neg\phi &\Leftrightarrow& \N,v\nVdash \phi \\
\N,w\Vdash \phi\land \psi &\Leftrightarrow&\N,w\Vdash \phi \textrm{ and } \N,w\Vdash \psi \\
\N,w\Vdash \K_i \phi &\Leftrightarrow&\forall v: w\sim_i v$ implies $ \N,w\Vdash \phi  \\
\N,w\Vdash [\Ue]\phi &\Leftrightarrow&\forall v: w\rel{\Ue}v \textrm{ implies } \N,v\Vdash \phi \\
\hline
\end{array}
$
\end{center}
Recall that given an epistemic temporal model $\N$, $\N^-$ is its induced epistemic core. By definition, the two semantics coincide on \LEL\ formulas ($[a]$-free):
\begin{proposition}\label{prop.coEL}
For any \LEL\ formula $\phi$ and any pointed epistemic temporal model $\N,w$: $\N,w\Vdash\phi\iff\N^-,w\vDash\phi$.
\end{proposition} 
Now we induce a game from an epistemic temporal model.   
\begin{definition}\label{Def.GamePart} Let $\N=(W,\sim,\to, V)$ be any epistemic temporal model, the game \emph{induced by} $\N$ is a structure $$\G_{\N}\ =(W,R,O)$$ where $O=V$ and $R\subseteq W\times \Act \times W$ is a relation such that for all $s,t\in W$ and $a\in \Act$, $\left\langle s,a,t\right\rangle \in R$ iff $s\stackrel{a}{\rightarrow}t$.
\end{definition}
Next, we show that $\G_\N$ is indeed a game structure if $\N$ is normal.
\begin{lemma} \label{game part} Let $\N=(W,\sim,\to, V)$ be any normal epistemic temporal model, and let $\G_{\N}\ =(W,R,O)$ be the game induced by  $\N$. Then $\G_{\N}$ is a game structure w.r.t. $\Act$ and $\Ag$; furthermore, $e(s)=Act(s)$ for each $s\in W$.
\end{lemma}

\noindent \begin{proof}
Clearly $W$ is a non-empty set and bear in mind below that $O=V$.  Firstly we prove that $R$ is a partial function from $W\times \Act$ to $W$.
Let $a\in \Act$ and $s,t,t'\in W$ such that $s\rel{\Ue}t$ and $s\rel{\Ue}t'$. 
Since $\N$ is normal, by $(\textbf{Det})$, we have $t=t'$.
It follows that $R$ is a partial function from $W\times \Act$ to $W$. 

Secondly we show that if $R(s,a)$ is defined, then $\emptyset \subset Act(s)\subseteq A_i$.
Suppose $R(s,a)$ is defined for some $s\in W$, $i\in \Ag$ and $a\in A_i$. Then $Act(s)$ is non-empty. Because $\N$ is normal, by $(\textbf{Exturn})$ $Act(s)\subseteq A_j$ for some $j\in \Ag$. Since $a\in A_i$ and $a\in Act(s)$, we have $i=j$.

Then we demonstrate that for all $s\in W$, $i\in \Ag$ and $a\in A_i$, $p^a\in V(s,i)$ iff $R(s,a)$ is defined. Consider any $s\in W$, $i\in \Ag$ and $a\in A_i$. By $(\textbf{Info})$ $p^a\in V(s,i)$ iff $s\rel{\Ue} t$ for some $t\in W$, iff, by the definition of $R$, $\lr{s,a,t}\in R$, i.e.,\ $R(s,a)$ is defined. 

Since we have already proven that $\G_{\N}$ is indeed a game structure, for each $s\in W$, $e(s)$ is well-defined. Thus lastly we can show that $e(s)=Act(s)$. But this is straightforward according to the definition of $e(s)$ and $Act(s)$.
\end{proof}
There is also an epistemic model hidden in an epistemic temporal model. 
\begin{definition}
Let $\N=(W,\sim,\to, V)$ be any epistemic temporal model, the \emph{epistemic part induced by} $\N$ is a tuple $\M_{\N}=\lr{W,f,\sim,V}$ where $f$ is the identity function on $W$.
\end{definition}

\begin{lemma}\label{epiPrt}
Let $\N=(W,\sim,\to, V)$ be any normal epistemic temporal model, and let $\M_{\N}=\lr{W,f,\sim,V}$ be the epistemic part induced by $\N$. Then $\M_{\N}$ is an epistemic model w.r.t.\ the game structure induced by $\N$.
\end{lemma}
\begin{proof}
By Lemma \ref{game part} the game induced by $\N$ is a game structure. We only need to prove that for all $s,t\in W$, and $i\in\Ag$, $V(s,i)=V(f(s),i)$, and that $s\sim_{i}t$ only if $V(f(s),i)=V(f(t),i)$.
Let $s,t\in W$, and $i\in\Ag$. Because $f$ is an identity function, we have $V(s,i)=V(f(s),i)$. Furthermore by  $(\textbf{Ke})$ if $s\sim_{i}t$, then $V(s,i)=V(t,i)$, which implies that $V(f(s),i)=V(f(t),i)$.
\end{proof}

The following lemma is crucial to the later result which connects the two semantics.
\begin{lemma}\label{Lem.Bis}
Let $\N=(W,\sim,\to, V)$ be any normal epistemic temporal model, $\G\ =(W,R,O)$ be the game induced by $\N$, let $\M=\lr{W,f,\sim,V}$ be the epistemic part  induced by $\N$. Then $s\rel{\Ue}t$ in $\N$ only if 
$$\M\otimes\U,(s,a)\bis_{\G}\M,t.$$ 
\end{lemma}

\noindent\begin{proof}
It is useful to note that for all $w,u\in W$ and $b\in \Act$, $w\rel{\Uf} u$ in $\N$ iff $wR_b u$ in $\G$, which we will use repeatedly in the following proof. By Lemma \ref{game part}, and Lemma \ref{epiPrt}, $\G$ and $\M$ are game structures and epistemic model w.r.t.\ $\G$.
Let $s\rel{\Ue}t$, let $\M\otimes\U=(W',f',\sim',V')$, 
and let $Z$ be the binary relation between $\M\otimes\U$ and $\M$ such that $(w,b)Zu$ iff $wR_{\Uf}u$ in $\G$.
Since $s\rel{\Ue}t$, we have $sR_{\Ue}t$. It follows that $Z$ is nonempty.

Let $w,u\in W$ and $b\in \Act$ such that $wR_{\Uf}u$.
We now check the three conditions of $\G$-bisimulation bearing in mind that $O=V$.

For $\textbf{Inv}$. First we have $f'(w,b)=R(f(w),b)$. 
Since $wR_b u$, $R(w,b)=u$, and because $f$ is an identity function, we have $$f'(w,b)=R(f(w),b)=R(w,b)=u=f(u).$$

For $\mathbf{Zig}$. Suppose $(w,b)\sim'_{i}(w',b')$ in $\M\otimes\U$ for some $i\in \Ag$. Let $u'$ be the state such that $w'R_{\Uf'}u'$. 
By the definition of $\sim'$, we have $w\sim_{i}w'$, $b\tot_{i}b'$ and $V(R(f(w),b),i)=V(R(f(w'),b'),i)$.
Because $R(w,\Uf)=u$, $R(w',\Uf')=u'$ and that $f$ is an identity function, we can get that $V(u,i)=V(u',i)$. Hence we have $$w\rel{\Uf}u,w\rel{\Uf'}u',w\sim_{i}w',b\tot_{i}b',V(u,i)=V(u',i),$$ and then by ($\NMp$), we obtain $u\sim_i u'$.

For $\mathbf{Zag}$. Suppose $u\sim_i u'$ for some $i\in \Ag$ and $u'\in W$. 
Then by $(\textbf{Ke})$ $V(u,i)=V(u',i)$.
Furthermore, by ($\PRp$) there is a $w'\in W$ and a $b'\in \Act$ such that $b\tot_i b'$, $w\sim_i w'$ and $w'\rel{\Uf'}u'$. 
Because $R(w,b)=u$, $R(w',b')=u'$ and $f$ is an identity function, we can acquire that $V(R(f(w),b),i)=V(R(f(w'),b'),i)$, and then by the definition of $\sim'$ we have $(w,b)\sim'_{i}(w',b')$.

Therefore $\M\otimes\U,(s,a)$ and $\M,t$ are $\G$-bisimilar.
\end{proof}
\subsection{The equivalence of the two semantics}
Now we are ready to establish the equivalence of the two semantics in two cases demonstrated by the following two theorems: one between a normal epistemic temporal model and its induced epistemic model, and another is between an epistemic model and its generated run model. It then follows that the valid formulas w.r.t.\ these two semantics are the same. 
\begin{theorem}\label{TH.Eqv}
Let $\N=(W,\sim,\to, V)$, let $\G=(S,R,O)$ be the game structure induced by $\N$, let $\M=\lr{W,f,\sim,V}$ be the epistemic model induced by $\N$. Then for any $\LDEL$ formula $\phi$ and any $s\in W$ $$\N,s\Vdash\phi \Longleftrightarrow \M,s\vDash\phi.$$
\end{theorem}
\begin{proof}
We prove by induction on the structure of the formulas.
The cases for Boolean combinations and $\K_i \phi$ are trivial due to Proposition \ref{prop.coEL} and the induction hypothesis. 

For the case $[\Ue]\phi$. 
On one hand, $\N,s\Vdash [\Ue] \phi$ iff for each $t\in W$, $s\rel{\Ue} t$ implies $\N,t\Vdash \phi$.
On the other hand $\M,s\vDash [\Ue] \phi$ iff $a\in e(f(w))$ implies $\M\otimes\U,(s,a)\vDash \phi$. Moreover, because $f$ is an identity function and by Lemma \ref{game part}, $e(s)=Act(s)$, $\M,s\vDash [\Ue] \phi$ is equivalent to that  $a\in Act(s)$ implies $\M\otimes\U,(s,a)\vDash \phi$.
It remains to show that
\begin{equation}
\textrm{ for each } t\in W, s\rel{\Ue} t \textrm{ implies } \N,t\Vdash \phi \textrm{ iff } a\in Act(s) \textrm{ implies } \M\otimes\U,(s,a)\vDash \phi.\label{Eqv.eq}
\end{equation}
Observe that $\Ue\notin Act(s)$ iff there is no $t\in W$ such that $s\rel{\Ue}t$.
Suppose there is no $t\in W$ such that $s\rel{\Ue}t$. Then $a\notin Act(s)$. It follows that both sides of the equivalence (\ref{Eqv.eq}) vacuously hold, and hence (\ref{Eqv.eq}) holds.
Suppose there is a $t\in W$ such that $s\rel{\Ue}t$. 
Then $a\in Act(s)$.
Since $\N$ is normal, by $(\textbf{Det})$ $t$ is the only state such that $s\rel{\Ue}t$.
By I.H., $\N,t\Vdash \phi$ iff $\M,t\vDash \phi$.
Using Lemma \ref{Lem.Bis}, we obtain that $\M\otimes\U,(s,a)\bis_{\G}\M,t$, and by Lemma \ref{Lem.TrPrsv} $\LDEL$ formulas are invariant w.r.t.\ $\vDash$, we have $\M,t\vDash \phi$ iff $\M\otimes\U,(s,a)\vDash \phi$. 
Therefore $\N,s\Vdash [\Ue]\phi$ iff $\M,s\vDash [\Ue]\phi$.
\end{proof}

\begin{theorem}\label{Th.UptoNem}
Let $\G$ be any game structure, $\M=(W,f,\sim,V)$ be any epistemic model, and $\F^{\otimes}(\M)=(W',\sim',\to',V')$ be the run model generated by repeatedly updating $\U$. Then for each $\LDEL$ formula $\phi$ and each $w\in W$, $$\M,w\vDash\phi \textrm{ iff } \F^{\otimes}(\M),(w)\Vdash\phi.$$
\end{theorem}
\begin{proof}
We proof by induction on the construction of $\phi$, and we only show the case $\phi=[\Ue]\psi$. On one hand, $\M,w\vDash\phi$ iff $\Ue\in e(f(w))$ implies $\M\otimes\U,(w,a)\vDash\psi$. On the other hand, $\F^{\otimes}(\M),(w)\Vdash\phi$ iff for each $v\in W'$, $(w)\rel{\Ue}v$ implies $\F^{\otimes}(\M),v\Vdash\psi$, iff, by determinacy, $(w)\rel{\Ue}v$ for some $v\in W'$ implies $\F^{\otimes}(\M),v\Vdash\psi$. Thus we have to demonstrate that \\

$\Ue\in e(f(w))\textrm{ implies }\M\otimes\U,(w,a)\vDash\psi\textrm{ iff } (w)\rel{\Ue}v\textrm{ for some }v\in W' \\ \textrm{ implies }\F^{\otimes}(\M),v\Vdash\psi.$ \hfill (2)

Note that for each $u\in W$ and $\Uf\in \Act$, $\Uf\in e(f(u))$ iff $(u)\rel{\Uf}u'$ for some $u'\in W'$. If $\Ue\notin e(f(w))$, then there is no $v\in W'$ such that $(w)\rel{\Ue}v$, which implies that both sides of the equivalence (2) are vacuously true, and hence (2) holds.

Suppose $\Ue\in e(f(w))$. Then $(w)\rel{\Ue}v$ for some $v\in W'$, and clearly $v=(w,a)$. It remains to show that $$\M\otimes\U,(w,a)\vDash\psi\textrm{ iff }\F^{\otimes}(\M),(w,a)\Vdash\psi.$$ We prove this by showing that $\F^{\otimes}(\M\otimes\U)$ is a generated submodel of $\F^{\otimes}(\M)$, and then using induction hypothesis.

We claim that $\F^{\otimes}(\M\otimes\U)$ is the submodel of $\F^{\otimes}(\M)$ generated from $B=\{(u,a_i,\dots,a_k)\in W'\mid k=1\}$. It is easy to check that it is indeed the case, and $(w,a)\in B$. As is well known in modal logic that each state in generated submodel satisfies the same formula in original model. It follows that $\F^{\otimes}(\M),(w,a)\Vdash\psi$ iff $\F^{\otimes}(\M\otimes\U),(w,a)\Vdash\psi$, and this completes our proof.
\end{proof}

Based on the above theorems, we can establish the equivalence of the two semantics modulo validities, which will play a crucial role in the later proof of the completeness of our proof system. Let $\mathbb{C}$ be the class of normal epistemic temporal models.

\begin{theorem}\label{Co.TrPrsv}
For each \LDEL\ formula $\phi$, $\vDash \phi$ iff $\mathbb{C}\Vdash\phi$.
\end{theorem}

\noindent\begin{proof}
Suppose $\mathbb{C}\nVdash\phi$ for some \LDEL\ formula $\phi$. Then there is an normal epistemic temporal model $\N\in \mathbb{C}$ such that $\N\nVdash\phi$. By Theorem \ref{TH.Eqv} there is an epistemic model $\M$ such that $\M\nvDash\phi$. Hence $\nvDash\phi$.

Suppose $\nvDash\phi$. Then there is an epistemic model $\M$ such that $\M\nvDash\phi$. By Theorem \ref{Th.UptoNem} $\F^{\otimes}(\M)\nvDash\phi$, and hence using Proposition \ref{Prop.RunisNem}, $\F^{\otimes}(\M)\in \mathbb{C}$. Therefore $\mathbb{C}\nVdash\phi$.
\end{proof}

\medskip

These two semantics and their correspondence provide us two perspectives to understand our logic: the constructive DEL-like perspective and the global ETL-like perspective. In the next section, we will see that the connection of these two perspectives also play a technical role: we will use the correspondence to give a complete axiomatization of our logic. 

\section{A complete axiomatization}
\label{sec.axiom}
In this section we give a complete axiomatization of the dynamic epistemic logic we proposed based on game structures and player assumptions. Due to theorem \ref{Co.TrPrsv}, we only need to show that the proof system is complete w.r.t.\ the $\Vdash$ semantics over normal epistemic temporal models. Namely the same proof system is complete for two different semantics. In fact, it is impossible to axiomatize our logic using the standard reduction method for \DEL: the action modality is clearly not reducible to the purely epistemic fragment. The techniques we use are   developed in \cite{WC12,WA13}. 

Note that although our logical approach in this paper is mostly semantics-driven, a proof system with intuitive axioms can help us to see clearly what we have assumed in our framework and to understand our own work better. 

\subsection{A proof system \SELA}
Recall that we fix an observation model $\U$. The following proof system also depends on the given $\U$. 
{\small \begin{center}
\begin{tabular}{lclc}
\multicolumn{4}{c}{System \SELA}\\
\multicolumn{2}{l}{\textbf{Axioms}}&\textbf{Rules}&\\
\TAUT & \tr{all the axioms of propositional logic}&\MP & $\dfrac{\phi,\phi\to\psi}{\psi}$\\
$\DISTK_i$ & $\K_i(\phi\to \psi)\to (\K_i\phi\to \K_i\psi)$&$\GENK_i$ &$\dfrac{\phi}{\K_i\phi}$\\
\DIST & $[a](\phi\to \psi)\to ([a]\phi\to [a]\psi)$ &$\GEN{a}$ &$\dfrac{\phi}{[a]\phi} $\\
 \AxTr& $\K_i\phi\to \phi $ &\phantom{$\dfrac{\phi}{\K\phi}$}\\
 \AxTrans& $\K_i\phi \to \K_i\K_i\phi$&\phantom{$\dfrac{\phi}{\K\phi}$}\\
 \AxEuc& $\neg \K_i\phi\to \K_i\neg \K_i\phi$&\phantom{$\dfrac{\phi}{\K\phi}$}\\
\NM &$\hat{\K_i}\langle b\rangle(\phi\wedge\phi^{o}_i)\rightarrow [a](\phi^{o}_i \rightarrow\hat{\K}_i\phi))$ \text{ (if $a\tot_i b$)}& \phantom{$\dfrac{\phi}{\K\phi}$}\\
\PR &$\lr{\Ue}\hat{\K}_i\phi\to \bigvee_{b:a\tot_i b}\hat{\K}_i\lr{\Uf}\phi$ && \phantom{$\dfrac{\phi}{\Box\phi}$}\\
 \DET &$\lr{a}\phi\to [a]\phi$&&$\phantom{\dfrac{\phi(p)}{\phi(\psi)}}$ \\
 \EXTURN &$ \TURN_i\to \bigwedge_{j\not=i}\neg \TURN_j$ &&\phantom{$\dfrac{\phi}{\K\phi}$}\\
 \INFO &$ \lr{a}\top\lra p^a_i$ \text{ (if $a\in A_i$)}& \phantom{$\dfrac{\phi}{\K\phi}$}\\
 \KE &$(p_i\to \K_ip_i) \land (\neg p_i\to \K_i\neg p_i)$&&\phantom{$\dfrac{\phi(p)}{\phi(\psi)}$} \\
\end{tabular}
\end{center}
}
\noindent where $a$ ranges over $\Act$, $p, q$ range over $\BP$ and  in the clause of \NM, $\phi^{o}_i$ ranges over the following set of formulas which exhausts every possible combination of information pieces for $i$ (note that $\BP$ is finite):
$$\{(\bigwedge_{p\in \B}p_i) \land (\bigwedge_{p\not\in \B}\neg p_i) \mid \B\subseteq \BP\}.$$
Therefore, $\NM$ is a finite schema. Since $\Act$ is finite, \PR\ is a finite schema. Note that if $\U$ consists of public actions only ($\tot_i $ are all self-loops), then we can replace $b$ with $a$ in \NM\ and \PR. Moreover, the system is clearly not closed under substitutions due to the special proposition letters. 
\medskip

Besides the `usual suspects' for modal logic like \DIST\ and \GEN{a}, axiom schemata $\NM$ and $\PR$ state two important properties about the players in our framework. Note that the general forms of the two axioms are $\hat{\K}\lr{\cdot}\phi\to [\cdot]\hat{\K}\phi$ and $\lr{\cdot}\hat{\K}\phi\to \hat{\K}\lr{\cdot}\phi$ respectively, which can be massaged into $\lr{\cdot}\K\phi\to \K[\cdot]\phi$ and $\K[\cdot]\phi\to [\cdot]\K\phi$. Then they roughly say ``if I know it now then I knew it would hold before the action'' (no miracles) and ``if I expect to know then I will indeed know it after the action'' (perfect recall). The formulas $\phi^{o}_i$ are used to control the condition under which no miracles can happen. Other axioms are quite straightforward; $\DET$ says that the actions are deterministic; \EXTURN\ says that each state can only belong to one player; \INFO\ gives the meaning to information pieces $p_i^a$; \KE\ says that the players always know the information given to them by the game.







\medskip

It is not hard to establish the soundness of the system. 
\begin{theorem}[Soundness]\label{Th.Sound}
The system $\SELA$ is sound w.r.t. $\vDash$.   
\end{theorem}
\begin{proof}
We only show the cases for axioms $\DET, \EXTURN, \INFO, \KE, \NM$ and $\PR$. Others are standard. Let $\M=(W, f, \sim, V)$ be any epistemic model w.r.t.\ a game structure $\G$, and let $w\in W$. 

For $\DET$, if $\M,w\vDash\lr{\Ue}\phi$, then $\Ue\in e(f(w))$ and $\M\otimes\U,(w,\Ue)\vDash\phi$, which implies $\M,w\vDash[\Ue]\phi$.

For $\EXTURN$. Recall that $\TURN_i=\bigvee_{\Ue\in A_i}\lr{\Ue}\top$ for each $\in\Ag$. Suppose $\M,w\vDash\TURN_i$. Then $\M,w\vDash\lr{\Ue}\top$ for some $\Ue\in A_i$, which implies $\Ue\in e(f(w))$. Using our requirement for $R_\G$ in Definition $\ref{Df.Game}$, we have $\Uf\notin e(f(w))$ for each $\Uf\notin A_i$. It follows that $\M,w\vDash\neg[\Uf]\top$ for each $\Uf\notin A_i$. Thus $\M,w\vDash\bigwedge_{j\neq i}\TURN_j$.

For $\INFO$. For all $i\in \Ag$ and $\Ue\in A_i$, $\M,w\vDash\lr{\Ue}\top$ iff $\Ue\in e(f(w))$ iff $R(f(w),\Ue)$ is defined, iff, by our definition about $O_\G$ in Definition $\ref{Df.Game}$, $p^a\in O_\G(f(w),i)=V(w,i)$, iff $\M,w\vDash p^a_i$.

For $\KE$. 
Let $p\in \BP$ and $i\in\Ag$. Suppose $\M,w\vDash p_i$. Then $p\in V(w,i)=O_\G(f(w),i)$. Consider any $u\in W$ with $w\sim_i u$. Using Definition \ref{Df.epi}, we have $V(u,i)=O_\G(f(u),i)=O_\G(f(w),i)$, which implies $p\in V(u,i)$, i.e.,\ $\M,u\vDash p_i$. It follows that $\M,w\vDash \K_i p_i$. Similarly we can prove 
that $\M,w\vDash \neg p_i$ only if $\M,w\vDash \K_i \neg p_i$.

For $\NM$. Suppose $\M,w\vDash\hat{\K}_i\lr{\Uf}(\phi\wedge\phi^o_i)$ for some $i\in\Ag$, some formula $\phi$, $\phi^o_i$ and some $\Uf\in\Act$. Then there are $u\in W$ and $\Uf\in \Act$ such that $w\sim_i u$, $\Uf\in e(f(u))$ and $\M\otimes\U,(u,\Uf)\vDash\phi\wedge\phi^o_i$. Recall that $\phi^o_i=(\bigwedge_{p\in \B}p_i) \land (\bigwedge_{p\not\in \B}\neg p_i)$ for some $\B\subseteq\BP$. Now consider any $\Ue\in\Act$ with $\Ue\tot_i \Uf$. If $\Ue\notin e(f(w))$, we vacuously have $\M,w\vDash [\Ue](\phi^o_i\rightarrow \hat{\K_i}\phi)$. If $\Ue\in e(f(w))$ and $\M\otimes\N,(w,\Ue)\vDash\phi^o_i$, then using $\M\otimes\U,(u,\Uf)\vDash\phi^o_i$, $$O_\G(R(f(w),\Ue),i)=V((w,\Ue),i)=V((u,\Uf),i)=O_\G(R(f(u),\Uf),i),$$ which implies, together with $\Ue\in e(f(w))$ and $w\sim_i u$, that $(w,\Uf)\sim'_i(u,\Ue)$, where $\sim'$ is the epsitemic relation in $\M\otimes\U$. Since $\M\otimes\U,(u,\Uf)\vDash\phi$, it follows that $\M\otimes\U,(w,\Ue)\vDash\hat{\K_i}\phi$. Therefore $\M\otimes\U,(w,\Ue)\vDash\phi^o_i\rightarrow\hat{\K_i}\phi$, and then $\M,w\vDash [\Ue](\phi^o_i\rightarrow\hat{\K_i}\phi)$.

For $\PR$. Suppose $\M,w\vDash\lr{a}\hat{\K_i}\phi$ for some $\Ue\in \Act$, $i\in \Ag$ and some formula $\phi$. Then $\M\otimes\U,(w,\Ue)\vDash\hat{\K_i}\phi$, and hence there are $u\in W$ and $\Uf\in \Act$ such that $\Uf\in e(f(u))$, $(w,\Ue)\sim_i'(u,\Uf)$ and $\M\otimes\U,(u,\Uf)\vDash\phi$. It follows that $\M,u\vDash\lr{\Uf}\phi$. By definition \ref{df:execution}, we have $w\sim_i u$ and $\Ue\tot_i \Uf$. Thus $\M,w\vDash\hat{\K_i}\lr{\Uf}\phi$, which implies that $\M,w\vDash\bigvee_{\Uf:\Uf\tot_i\Ue}\hat{\K_i}\lr{\Uf}\phi$.
\end{proof}
\subsection{Completeness}

For completeness, we use the strategy proposed in \cite{WC12,WA13}. The idea is to turn the dynamics (updates) into transitions in a two dimensional Kripke model and use some special properties to characterize the updates such that the dynamics can be flattened out under these properties. The general strategy is as follows: 


$${\vDash\phi}\;{\Longrightarrow}\;{\mathbb{C}\Vdash\phi}\;{\Longrightarrow}\;{\vdash_{\SELA}\phi}\textrm{ for each }\LDEL\textrm{ formula }\phi.$$ 
The hardest part (the first step) has been solved in Theorem \ref{Co.TrPrsv}. In the following we prove the second step, namely the completeness w.r.t.\ the normal epistemic temporal models under $\Vdash$.\footnote{It is not that hard to see that the completeness should work in principle, based on the Sahlqvist completeness theorem: the axioms are mostly Sahlqvist formulas with a few exceptions which are not closed under substitution (they specify properties of the \textit{models} rather than frames). We will spell out the details in a proof based on the canonical model. }

\medskip

We use $\Gamma,\Lambda\ldots$ for sets of \LDEL\ formulas. As usual, we say a set of \LDEL\ formulas $\Gamma$ is \textit{maximally consistent} if it is $\SELA$-consistent, and any set properly containing $\Gamma$ is $\SELA$-inconsistent. 

\begin{definition}[Cononical model]
A canonical model $\N^C$ is a structure $$(W^C,\{\simeq_i\mid i\in \Ag\},\{\stackrel{\Ue}\rightarrowtail\mid \Ue\in \Act\},V^C)$$ where
 \begin{itemize}
  \item $W^C$ is the set of all maximally consistent sets;
  \item for all $w,u\in W^C$ and $i\in \Ag$, $w\simeq_i u$ iff $\{\phi\mid\K_i\phi\in w\}\subseteq \{\phi\mid \phi\in u\}$;
  \item for all $w,u\in W^C$ and $a\in \Act$,$w \stackrel{\Ue}\rightarrowtail u$ iff $\{\phi\mid [\Ue]\phi\in w\}\subseteq \{\phi\mid \phi\in u\}$; 
  \item for all $w\in W^C$, $i\in \Ag$ and $p\in\BP$, $w\in V^C(p,i)$ iff $p_i\in w$.
 \end{itemize}
\end{definition}
The following are routine exercises for normal modal logic (cf. e.g., \cite{mlbook}). 
\begin{lemma}\label{lem.truth}
Let  $\N^C=(W^C,\simeq,\rightarrowtail,V^C)$ be a canonical model. Then for all $w,u\in W^C$, $i\in \Ag$ and $\Ue\in \Act$:
\begin{itemize}
 \item $w\simeq_i u$ iff $\{\phi\mid\phi\in  u\}\subseteq\{\phi\mid\hat\K_i\phi\in w\}$;
 \item $\simeq_i$ is an equivalence relation for each $i\in\Ag$;
 \item $w\stackrel{\Ue}\rightarrowtail u$ iff $\{\phi\mid\phi\in u\}\subseteq\{\phi\mid\lr{\Ue}\phi\in w\}$;
 \item (truth lemma) for each \LDEL\ formula $\phi$, $\N^C,w\Vdash \phi$ iff $\phi\in w$.
\end{itemize}
\end{lemma}

The only missing part is to show that the canonical model is indeed a normal epistemic temporal model w.r.t.\ $\U$.

\begin{lemma}\label{Prop.Can}
Given $\U$, the canonical model $\N^C=(W^C,\simeq,\rightarrowtail,V^C)$ is a normal epistemic temporal model.
\end{lemma}

\noindent\begin{proof}
We only show that $\N^C$ satisfies the condition $(\textbf{Det})$, $(\textbf{Exturn})$, $(\textbf{Info})$, $(\textbf{Ke})$, $(\textbf{Nm})$ and $(\textbf{Pr})$. Let $w,u,v\in W^C$, $i\in \Ag$ and $\Ue\in\Act$.

For $(\textbf{Det})$, suppose $w \stackrel{\Ue}\rightarrowtail u$, $u \stackrel{\Ue}\rightarrowtail v$, and $u\neq v$. Then there is a $\phi\in u$ such that $\neg\phi\in v$, which implies $\lr{a}\phi\wedge\lr{a}\neg\phi\in w$. However, this is impossible, because by axiom $\DET$, $\lr{a}\phi\rightarrow [a]\phi\in w$, and then $[a]\phi\in w$, a contradiction. 

For $(\textbf{Exturn})$, suppose $Act(w)\neq\emptyset$. Then there is a $j\in \Ag$ and a $\Uf\in A_j$ such that $\Uf\in Act(w)$, which implies $\lr{\Uf}\top\in w$. Recall that $\TURN_i=\bigvee_{c\in A_i}\lr{c}\top$, and hence $\vdash\lr{\Uf}\top\rightarrow\TURN_j$. It follows, by axiom $\EXTURN$, that $\bigwedge_{k\neq j}\neg\TURN_k\in w$, i.e.,\ $\neg \lr{c}\top\in w$ for each $c\in \Act$ such that $c\notin Act_j$. Thus $Act(w)\subseteq A_j$.

For $(\textbf{Info})$, consider any $\Uf\in A_i$. $p^b\in V^C(w,i)$ iff $p^b_i\in w$ iff, by axiom $\INFO$, $\lr{\Uf}\top\in w$ iff $w\stackrel{\Uf}\rightarrowtail u'$ for some $u'\in W^C$. 

For $(\textbf{Ke})$, suppose $w\simeq_i u$ and $V^C(w,i)\neq V^C(u,i)$. Then $V^C(w,i)$ and $V^C(u,i)$ do not contain all the same information pieces. Without losing any generality suppose there is a $p\in\BP$ such that $p\in V^C(w,i)$ but $p\notin V^C(u,i)$. Then $p_i\in w$ and $p_i\notin u$. Using axiom $\KE$ we have $p_i\rightarrow\K p_i\in w$, which implies $\K p_i\in w$. Since $w\simeq_i u$, we can obtain that $p_i\in u$, a contradiction!

For $(\textbf{Nm})$, suppose $w\simeq_i u$ and $u\stackrel{\Ue}\rightarrowtail v$. Let $\B\subseteq\BP$ be the set of information pieces such that $p\in \B$ iff $p_i\in v$, and let $$\psi=(\bigwedge_{p\in \B}p_i)\wedge (\bigwedge_{p\notin \B}\neg p_i).$$ Then $\psi\in v$.
We can check that for each $s\in W^C$, $\psi\in s$ iff $V^C(s,i)=V^C(v,i)$. Now consider any $s\in W^C$ and $\Uf\in \Act$ such that $w\stackrel{\Uf}\rightarrowtail s$,$\Ue\tot_i\Uf$ and $V^C(s,i)=V^C(v,i)$, and then $\psi\in s$. We claim that 
\begin{equation}
\{\phi\mid\phi\in v\}\subseteq\{\phi\mid\hat{\K}_i\phi\in s\}\label{Can.Nm}
\end{equation}
and hence $s\simeq_i v$. Suppose $\phi\in v$. Then $\lr{a}(\phi\wedge\psi)\in u$, which implies $\hat{\K}_i\lr{a}(\phi\wedge\psi)\in w$. By axiom $\NM$ and $\Ue\tot_i\Uf$, we have $[\Uf](\psi\rightarrow\hat{\K}_i\phi)\in w$, and because $w\stackrel{\Uf}\rightarrowtail s$, we can get that $(\psi\rightarrow\hat{\K}_i\phi)\in s$. Since $\psi\in s$, we obtain  $\hat{\K}_i\phi\in s$. Therefore $(\ref{Can.Nm})$ holds, and this completes our proof for $(\textbf{Nm})$.

For $(\textbf{Pr})$, suppose $w\stackrel{\Ue}\rightarrowtail u$ and $u\simeq_i v$. Let $\Gamma=\{\phi\mid\K_i\phi\in w\}$, and for each $\Uf\in\Act$, let $\Lambda_b=\{\lr{\Uf}\phi\mid\phi\in v\}$. We show that 
\begin{equation}
\textrm{there is a } \Uf\in\Act \textrm{ such that } \Uf\tot_i\Ue \textrm{ and } \Gamma \cup\Lambda_b \textrm{ is $\SELA$-consistent}.\label{Con.Pr}
\end{equation}
Note that by axiom $\DET$, for each finite sequence of formulas $\phi_1,\phi_2,\ldots,\phi_n$ and each $\Uf\in\Act$, 
$$\vdash\lr{\Uf}\phi_1\wedge\lr{\Uf}\phi_2\wedge\cdots\wedge\lr{\Uf}\phi_n\leftrightarrow\lr{\Uf}(\phi_1\wedge\phi_2\wedge\cdots\wedge\phi_n).$$
Hence for each $\Uf\in\Act$, if $\Gamma\cup\Lambda_b$ is $\SELA$-inconsistent, then there are formulas $\psi$ and $\phi$ such that $\psi\in\Gamma$, $\lr{\Uf}\phi\in\Lambda_b$ and $\vdash\lr{\Uf}\phi\rightarrow\neg\psi$. For each $\Uf\in \Act$, if $\Gamma\cup\Lambda_b$ is $\SELA$-inconsistent, then let $\psi_b$ and $\phi_b$ be such $\psi$ and $\phi$. Suppose towards contradiction that for each $\Uf\in\Act$, $\Uf\tot_i\Ue$ only if $\Gamma\cup\Lambda_b$ is $\SELA$-inconsistent. Because $\lr{\Uf}\phi_b\in\Lambda_b$ for each $\Uf$ with $\Uf\tot_i \Ue$, and by the definition of $\Lambda_b$, we have $\phi_b\in v$ for each $\Uf\tot_i \Ue$. Then 
$$\textrm{for each } c \textrm{ with } c\tot_i\Ue, \lr{c}\lambda\in\Lambda_c\textrm{ and }\vdash\lr{c}\lambda\rightarrow\lr{c}\phi_c,\textrm{ where }\lambda=\bigwedge_{b:b\tot_i\Ue}\phi_b.$$
Since $\vdash\lr{\Uf}\phi_b\rightarrow\neg\psi_b$ for each $b$ with $b\tot_i a$, it follows that
\begin{equation}
\vdash\bigwedge_{b:b\tot_i a}(\lr{\Uf}\lambda\rightarrow \neg\chi)\textrm{ where }\chi=\bigwedge_{b:b\tot_i a}\psi_b.\label{The.Pr}
\end{equation}
Note that by the our assumption about $\Gamma$, $\chi\in\Gamma$. Since $\lambda\in v$, $w\stackrel{\Ue}\rightarrowtail u$ and $u\simeq_i v$, we have $\lr{\Ue}\hat\K_i\lambda\in w$, which implies, by axiom $\PR$, $\bigvee_{b:b\tot_i a}\hat\K_i\lr{\Uf}\lambda\in w$. By $(\ref{The.Pr})$ we have $\bigwedge_{b:b\tot_i a}\K_i(\lr{\Uf}\lambda\rightarrow \neg\chi)\in w$, and then $\hat\K_i\neg\chi\in w$, but this is impossible, for $\chi\in\Gamma$, i.e.,\ $\K_i\chi\in w$. Therefore (\ref{Con.Pr}) holds for some $b$, and hence there is a maximally consistent set $\Lambda$ such that $\Gamma\cup\Lambda_b\subseteq\Lambda$, which implies $w\simeq_i \Lambda$ and $\Lambda\stackrel{\Uf}\rightarrowtail v$.
\end{proof}

\begin{theorem}\label{Th.Com}
$\SELA$ is sound and complete w.r.t. $\lr{\LDEL,\mathbb{C},\Vdash}$.
\end{theorem}
\begin{proof}
Using Theorem \ref{Th.Sound} and Corollary \ref{Co.TrPrsv}, we have the soundness of $\SELA$. By Lemmas \ref{lem.truth} and \ref{Prop.Can} we obtain that $\SELA$ is complete.
\end{proof}

From Theorem \ref{Th.Sound}, Corollary \ref{Co.TrPrsv} and Theorem \ref{Th.Com} it follows:

\begin{theorem}
$\SELA$ is sound and complete w.r.t. $\lr{\LDEL, \mathbb{M},\vDash}$ where $\mathbb{M}$ is the collection all the epistemic models. 
\end{theorem}


\section{Conclusions and future work}
\label{sec.con}
In this paper, we present an attempt to formally separate the game rules from the game runs, and to compute the latter from the earlier with extra assumptions of the players in terms of initial uncertainty, observability and knowledge update mechanism. In this way we can explicitly specify what information is given by the game and what information is obtained by the players during the plays, thus clarifying different interpretations of information sets in imperfect information games. We show that under certain (strong) conditions on the game structures, the distinction of game rules and runs is negligible, thus also explaining why it is sometimes harmless to mix the two in practice. We propose a logic to reason about game runs under two equivalent semantics based on the ideas of  \DEL\ and \ETL\ respectively. We also give a complete axiomatization of the logic to syntactically capture the assumptions that we made in the framework. 

\medskip 

As we mentioned in the introduction, our work is related to the GDL-II language proposed by  \cite{SchiffelT14}, which deserves a more detailed comparison since not much has been said connecting GDL with dynamic epistemic logic. GDL-II is a simple yet elegant language of specifying game rules with imperfect information. GDL-II is given a semantics to compute the corresponding game runs with uncertainty in a very similar way to the update product in this paper. In GDL-II, a player $i$ cannot distinguish two histories $h$ and $h'$ if $i$ \texttt{sees} (a special predicate) exactly the same information provided by the game along $h$ and $h'$ and the local actions in $h$ and $h'$ for $i$ are the same.\footnote{GDL-II allows concurrent actions where a global transition in game is induced by a tuple of local moves.} Under synchronicity, the indistinguishability at the next stage of the game runs only depends on the indistinguishability at the previous stage and the local actions, which implicitly assumes (variants of) perfect recall and no miracles. On the other hand, GDL-II does not distinguish the assumptions about the players and the information provided by the game and assume a unique starting state which is common knowledge. The initial uncertainty and the observational power of the players are encoded by game information and random moves. GDL-II models can be translated to various epistemic temporal models \cite{RuanT14,HuangRT13,RuanT12} where epistemic temporal languages can be used to talk about the properties of the games. 

\medskip

Although our approach employs specific designs about each component of the game  and the player assumptions, e.g., the update product $\otimes$ which presupposes synchronous perfect recall, we hope the readers can see the flexibility for our work to be adapted with the following features: 
\begin{itemize}
\item update mechanisms without perfect recall such as the ones discussed by  \cite{LG14};
\item explicit memory components in the epistemic models such as the sequences of actions or information pieces in \cite{Kaneko08}. Note that in our work the memory ability is implicitly characterized by the update mechanism thus to incorporate the more general memory components we need to revise the update mechanism accordingly with explicit memory checking;
\item the expected strategies to capture the rationality of the players, in order to strengthen the reasoning power of the players. The strategies may be encoded as well in the epistemic models similar to the local states with strategies in \cite{Halpern97Ab};
\item observation models with higher-order uncertainty, e.g., player 1 is sure what happened but he is not sure whether player 2 knows it. It can be handled  by copies of the same action with the help of a function from the `action copies' in the observation model to actions in $\Act$, just like the function $f$ in the epistemic model w.r.t.\ a game;
\item a logical language with ATL-like operators such as the ones used by  \cite{GH10} based on CIGS. 
\item a model checking algorithm based on game rules directly, like the one for a similar language in the setting of epistemic planning \cite{YLW15}. We do not need to compute the complete run model in order to check the truth value of the \LDEL\ formulas. However, if we strengthen the expressive power of the language with iteration or fixed point operators the model checking problem may become highly non-trivial (cf. e.g., \cite{AucherB13}).  
\end{itemize}
With some of those extensions in place, we may be able to discuss (new) solution concepts as explored in the interpreted system framework with knowledge-based programs by \cite{HalpernM07}.  

\medskip

Of course, there are many interesting problems to be discussed further. For example, although we try to separate the game rules from the player assumptions, sometimes the distinction is rather blurred. For example, game rules may also include assumptions about the observability of the players (e.g., they should not see others' cards), which can be encoded by the game information in our framework. Should we make them more explicit in the rules? Moreover, sometimes the initial uncertainty is caused by the game rules (e.g., drawing the cards from a shuffled deck), should we include it in the game rules instead of the player assumptions? Moreover, in our framework, each player gets some information even on the states which do not belong to him or her. However, does this extra information for the `currently irrelevant' players really play a role for reasonable solutions of the game? 

We close with a more technical question which we also leave for a future occasion: is our run model always regular modulo epistemic bisimulation? In the other words, will the computation of the next level start to cycle at some point? It is not the case in DEL in general, but our model and update are quite restricted which may give us some hope on the regularity of the computation.   













\bibliographystyle{apalike}
\bibliography{frtr}


\end{document}